\documentclass[11pt]{article}

\usepackage{bbold,bm,amsmath,amssymb,latexsym}
\usepackage[colorlinks = true, citecolor = blue]{hyperref}
\usepackage{bookmark}
\usepackage[numbers]{natbib}
\usepackage[all]{hypcap}
\usepackage{booktabs}
\usepackage{cite}{\tiny }
\usepackage{afterpage}
\usepackage{graphicx}
\usepackage{mdframed}
\usepackage[normalem]{ulem}
\usepackage[table]{xcolor}
\usepackage{makeidx}
\usepackage{cleveref}
\usepackage[centerlast]{caption}
\usepackage{float}
\usepackage[lined,boxed,commentsnumbered]{algorithm2e}
\usepackage[all]{xy}
\usepackage{tikz}
\usepackage{pgfplots}
\usetikzlibrary{arrows,intersections}
\usepackage{times}
\usepackage{float}
\usepackage{subfig}
\usepackage{authblk}
\usepackage{}
\usepackage[T1]{fontenc}

\newlength\aftertitskip     \newlength\beforetitskip
\newlength\interauthorskip  \newlength\aftermaketitskip

\pgfplotsset{compat=1.18}

\newcommand{\BlackBox}{\rule{1.5ex}{1.5ex}}  
\newenvironment{proof}{\par\noindent{\bf Proof\ }}{\hfill\BlackBox\\[2mm]}
 
\newtheorem{theorem}{Theorem}
\newtheorem{lemma}[theorem]{Lemma} 
\newtheorem{proposition}[theorem]{Proposition} 

\newtheorem{corollary}[theorem]{Corollary}
\newtheorem{definition}[theorem]{Definition}

\DeclareMathOperator*{\argmin}{arg\,min}
\DeclareMathOperator*{\argmax}{arg\,max}

\DeclareMathOperator{\sign}{sign}

\newcommand{\RR}{\protect\mathbb{R}}

\newcommand{\EE}{\protect\mathbb{E}}
\newcommand{\PP}{\protect\mathbb{P}}

\newcommand{\alg}{\protect\mathcal{A}}

\newcommand{\Fclass}{\protect\mathcal{F}}

\newcommand{\BoundedF}[2]{\protect{[-#2,#2]^{#1}}}
\newcommand{\Hilbert}{\protect\mathcal{H}}
\newcommand{\chull}{\protect\mathrm{co}}
\newcommand{\Span}{\protect\mathrm{span}}

\newcommand{\dist}{\raise.17ex\hbox{$\scriptstyle\mathtt{\sim}$}}

\newcommand{\risk}{\protect{\mathrm{Risk}}}

\newcommand{\regret}[1]{\protect{\mathrm{Regret}_{#1}}}

\newcommand{\norm}[1]{\left\lVert #1 \right\rVert}
\newcommand{\hnorm}[1]{\left\lVert #1 \right\rVert_\Hilbert}

\newcommand{\allsamples}{\cup_{n=1}^\infty (X\times Y)^n}
\newcommand{\ip}[2]{\left< #1, #2\right>}
\newcommand{\featuremap}[2]{#1 : #2 \rightarrow \Hilbert}
\newcommand{\llin}{\protect{\ell_{\mathrm{linear}}}}
\newcommand{\lmis}{\protect{\ell_{01}}}
\newcommand{\risklin}{\protect{\riskL{\mathrm{linear}}}}
\newcommand{\riskmis}{\protect{\riskL{01}}}

\newcommand{\lmargin}[1]{\ell_{#1}}

\newcommand{\riskL}[1]{\risk_{#1}}
\newcommand{\pred}[1]{[\![ #1 ]\!]}

\newcommand{\logdelta}{\log\left(\frac{1}{\delta}\right)}

\newcommand{\TwoVector}[2]{
	\left(\begin{array}{c}
		#1  \\ 
		#2  
	\end{array} \right)
}

\newcommand{\Tsymmetrc}{\left(\begin{array}{cc}
		1-\sigma & \sigma \\ 
		\sigma & 1-\sigma
	\end{array} \right)}

\newcommand{\Msymmetric}{\frac{1}{1- 2 \sigma}\left( \begin{array}{cc}
		1-\sigma & -\sigma \\ 
		-\sigma & 1-\sigma
	\end{array} \right)}

\title{
	Sparse Robust Classification via the Kernel Mean
}

\author[1]{Brendan van Rooyen}
\author[2]{Aditya Krishna Menon}
\author[3]{Robert C. Williamson}
\affil[1]{Decisions 360}
\affil[2]{Google Research, NY}
\affil[3]{University of T\"{u}bingen}

\date{}

\begin{document}
	
\maketitle

\begin{abstract}

Many leading classification algorithms output a classifier that is a weighted average of kernel evaluations. Optimizing these weights is a nontrivial problem that still attracts much research effort. Furthermore, explaining these methods to the uninitiated is a difficult task. Letting all the weights be equal leads to a conceptually simpler classification rule, one that requires little effort to motivate or explain, the mean. Here we explore the consistency, robustness and sparsification of this simple classification rule.

\end{abstract}

\section{Introduction}

In the problem of binary classification, the goal is to learn a classifier that accurately predicts the corresponding label of an observed instance. Given a sample $\{(x_i,y_i)\}_{i=1}^n$, many classification algorithms, such as the support vector machine, logistic regression, boosting (for a particular choice of weak learners) and so on, output a classifier of the form,
\begin{equation*}
f(x) = \sign\left( \sum\limits_{i=1}^n\alpha_i y_i K(x_i,x) \right),
\end{equation*}
with $\alpha_i \in \RR$ and $K(x,x')$ a function that measures the similarity of two instances $x$ and $x'$. Although there are many sophisticated methods that can optimize the weights $\alpha_i$, it is nevertheless a non-trivial problem that still attracts a lot of research effort. Furthermore, \emph{explaining} these methods to the uninitiated is a difficult task. Letting all $\alpha_i$ be equal leads to a conceptually simpler classification rule, one that requires little effort to motivate or explain: the mean classifier,
$$
f(x) = \sign\left(\frac{1}{n} \sum\limits_{i=i}^{n} y_i K(x_i,x)\right).
$$
The above is a simple and intuitive classification rule. It classifies by the total similarity to the previously observed positive and negative instances, with the most similar class the output of the classifier. It has been studied previously, for example, in chapter one of \citep{Scholkopf:2002} and further in \citep{Devroye1996,Servedio1999,Kalai2008,Balcan2008}. We will show that in addition to the obvious simplicity, this approach has some \emph{unique} advantages.
\\
\\
We argue for the mean classifier as follows:

\begin{itemize}
	\item We show that the mean classifier is the empirical risk minimizer for a classification-calibrated loss function (theorems \ref{Surrogate Regret Bound for Linear Loss} and \ref{Mean Classifier minimizes Linear Loss}).

	\item We explore the robustness properties of the mean classifier. We relate its noise tolerance to the \emph{margin for error} in the solution (theorem \ref{Margins and Approximation}). 
	
	\item In a certain sense, the mean classifier is the \emph{only} surrogate loss minimization method that is \emph{immune} to the effects of symmetric label noise (Theorem \ref{Uniqueness of Linear Loss}). Furthermore, we show how the mean classifier avoids the negative results outlined in \citep{Long2008}, which show that small amounts of label noise can break standard methods.
	
	\item We present other results beyond those for symmetric label noise (Section \ref{sec:beyond-symmetric-label-noise}).
	
	\item We show how a \emph{simple} sub-sampling scheme can be used to sparsely approximate \emph{any} kernel classifier, with provable approximation guarantees (Section \ref{sec:sparse-approximation-via-the-exploitation-of-clusters}).
	
	\item Finally, we present experiments corroborating the sparseness and robustness guarantees outlined in our theorems (Section \ref{sec:experimental-validation}).
	
\end{itemize}
The result is a conceptually simple algorithm for learning classifiers that is accurate, easily parallelized, robust, and firmly grounded in theory. All proofs are collected in the appendix \ref{sec:proofs-of-theorems-in-the-main-text}.

\section{Background Ideas}

Let $X$ be the instance space and $Y =\{-1,1\}$ the label space. A \emph{classifier} is a bounded function $f \in \RR^X$, with $f(x)$ the \emph{score} and  $\sign(f(x))$ the predicted label. A loss is a function $\ell : Y \times \RR \rightarrow \RR$. We will always assume $X$ to be a measure space with all respective classifiers and loss measurable functions. We measure the distance between classifiers via the \emph{supremum distance}, 
$$
\lVert f - f' \rVert_\infty = \sup_{x \in X} |f(x) - f'(x)|.
$$
For any Boolean predicate $p$, let $\pred{p(x)}$ be the function that returns $1$ if $p$ is true and $0$ otherwise. Define the misclassification loss $\lmis(y,v) = \pred{y v \leq 0}$. Note that $\lmis(y,0) = 1$ always. This non-standard presentation of misclassification loss will enhance the readability of many of the proofs. An output of zero can be viewed as \emph{abstaining} from choosing a label. Let $P \in \PP(X \times Y)$ be a distribution over instance label pairs and $S = \{(x_i, y_i)\}_{i=1}^{n}$ be a sample comprising of $n$ independent draws from $P$. For any loss, \emph{risk} and \emph{sample risk} of $f$ are defined as 
\begin{align*} 
\riskL{\ell}(P,f) := \EE_{(x,y) \dist P} \ell(y,f(x)) \ \text{and}\ & \riskL{\ell}(S,f) := \frac{1}{n} \sum\limits_{i = 1}^{n} \ell(y_i,f(x_i)), 
\end{align*}
respectively. Define the \emph{Bayes optimal} classifier and regret it to be
\begin{align*} 
f_{\ell, P} := \argmin_{f \in \RR^X} \riskL{\ell}(P,f) \ \text{and}\ & \regret{\ell}(P,f) = \riskL{\ell}(P,f) - \riskL{\ell}(P,f_{\ell, P}).
\end{align*}
respectively \footnote{We assume that an $\argmin$ exists, which will be the case for losses and function classes under consideration.}. The risk of the Bayes optimal classifier is the smallest possible risk under the assumption that the data is drawn from $P$. The regret measures the suboptimality of $f$. For misclassification loss, the Bayes optimal classifier takes a very simple form, $f_{\mathrm{01}, P}(x) = 1$ if $P(Y=1| X = x) \geq \frac{1}{2}$ and $-1$ otherwise. 
\\
\\
A \emph{classification algorithm} is a function, 
$$
\alg : \allsamples \rightarrow \RR^X,
$$ 
that given a training set $S$ outputs a classifier. Good classification algorithms should produce classifiers with low risk of misclassification. A naive classification algorithm proceeds via the direct minimization of, 
$$
\riskmis(S,f) := \riskL{\lmis}(S,f),
$$
with $f$ lying in some suitable large function class $\Fclass$. Even for a reasonably simple $\Fclass$, this approach is computationally infeasible. Many computationally feasible classification algorithms, such as the SVM, logistic regression, boosting (for a particular choice of weak learners) and so on proceed via minimizing a convex potential (or margin) loss function over a linear function class. 

\subsection{Linear Function Classes, Kernel Methods and Convex Potential Losses}

Linear and kernel methods \citep{Shawe-Taylor:2004,Scholkopf:2002} constitute a powerful class of machine learning techniques. They proceed by mapping the instances into a high (possibly infinite) dimensional space, before applying standard procedures from convex optimization to find a suitable classifier. The representer theorem \citep{kimeldorf1970correspondence,Scholkopf:2002} together with several recent algorithmic advances \citep{williams2001using,Rahimi:2007,shalev2011pegasos} provides computationally feasible means to apply kernel methods in practice.
\\
\\
Denote by $\Hilbert$ an abstract Hilbert space, with inner product $\ip{v_1}{v_2}_\Hilbert$ and norm $\norm{v}_\Hilbert = \sqrt{\ip{v}{v}_\Hilbert}$. When the Hilbert space is clear from context, we drop the subscript. In usual \emph{linear} approaches to machine learning, $\Hilbert = \RR^d$. The power of kernel methods comes from working with infinite dimensional $\Hilbert$. For a feature map $\featuremap{\phi}{X}$ define the linear function class,
$$
\Fclass_\phi := \left\{ f_\omega(x) = \langle \omega , \phi(x) \rangle : \omega \in \Hilbert \right\},
$$
and the bounded linear function class,
$$
\Fclass_\phi^r := \left\{ f_\omega(x) = \langle \omega , \phi(x) \rangle : \omega \in \Hilbert , \norm{\omega} \leq r \right\},
$$
with,
$$
K(x,x') = \ip{\phi(x)}{\phi(x')},
$$
the \emph{kernel} corresponding to $\phi$. We will assume throughout that the feature map is \emph{bounded}, $\norm{\phi(x)} \leq 1$ for all $x$. In the language of kernels, this ensures $K(x,x') \in [-1,1]$. By the Cauchy-Schwarz inequality $\Fclass_\phi^r \subseteq \BoundedF{X}{r}$. When convenient we identify $f_\omega$ with its weight vector $\omega$, and as shorthand write $\riskL{\ell}(P,\omega) := \riskL{\ell}(P, f_\omega)$. We call $\phi$ \emph{universal} \citep{steinwart2001influence,micchelli2006universal} if $\Fclass_\phi$ is dense in $\RR^X$. An example of a universal feature map is that associated with the Gaussian kernel,
$$
K(x,x') = \exp\left(-\frac{\norm{x - x'}_2^2}{2}\right), \ \forall x, x' \in \RR^d.
$$
As a surrogate to minimizing $\riskmis(P,f)$ over \emph{all} possible classifiers, standard approaches to learning classifiers choose a \emph{convex potential} loss function $\ell$ and return the classifier,
$$
f^* = \argmin_{f \in \Fclass_\phi^r} \riskL{\ell} (S, f).
$$
\begin{definition}
	A loss $\ell$ is a \emph{convex potential} if there exists a convex function $\psi : \RR \rightarrow \RR$ with $\psi(v) \geq 0$, $\psi'(0) < 0$ and $\lim_{v \rightarrow \infty} \psi(v) = 0$, with,
	$$
	\ell(y,v) = \psi(y v).
	$$
	
\end{definition}
The requirement that $\psi'(0) < 0$ ensures that all convex potential loss functions are \emph{classification calibrated}.

\begin{definition}
	A loss function $\ell$ is \emph{classification calibrated} \citep{Bartlett2006} if for all distributions $P$ and sequences of classifiers $f_n$,
	$$
	\regret{\ell}(P,f_n) \rightarrow 0 \implies \regret{\mathrm{01}}(P,f_n) \rightarrow 0.
	$$
\end{definition} 
All standard losses used in machine learning, that is, hinge, logistic, and exponential losses, are classified calibrated \citep{Bartlett2006}. The regularization parameter $r$ governs the trade-off between over-fitting versus small sample risk. Utilizing a universal kernel and allowing $r \rightarrow \infty$ as $n \rightarrow \infty$ yields a consistent algorithm for learning classifiers.
\\
\\
The representer theorem \citep{Scholkopf:2002,kimeldorf1970correspondence} states that $f^*$ has the form,
$$
f^*(x) = \sum_{i=1}^{n} \alpha_i K(x,x_i).
$$ 
We will explore the special case where $\alpha_i = \frac{y_i}{n}$, 
\begin{equation}\label{Kernel Mean solution}
f(x') = \frac{1}{n} \sum\limits_{i = 1}^{n} y_i K(x_i,x').
\end{equation}

\section{Why the Mean?}

The mean is not only an intuitively appealing classification rule, it also arises as the optimal classifier for the linear loss, considered previously in \citep{Reid2009b} and \citep{Sriperumbudur2009}. Let,
$$
\llin(y,v) := 1 - y v,\ v \in \RR.
$$ 
If $v \in \{-1,1\}$, then $\lmis(y,v) = \frac{1}{2}\llin(y,v)$. Allowing $v \in [-1,1]$ provides \emph{convexification} of misclassification loss. For  $v \in [-1,1]$, $\lmis(y,v) \leq \llin(y,v)$ . Furthermore, linear loss is classification calibrated.

\begin{lemma}[\citep{Steinwart:2008} theorem 2.31]\label{Surrogate Regret Bound for Linear Loss} 
For all distributions $P$ and for all $f \in \BoundedF{X}{1}$,
	$$
	\regret{\mathrm{01}}(P,f) \leq \regret{\mathrm{linear}}(P,f).
	$$
\end{lemma}
We include a proof for completeness. By a simple corollary of the lemma \ref{Surrogate Regret Bound for Linear Loss}, linear loss is classification calibrated, provided that we only work with classifiers $f \in \BoundedF{X}{1}$.  For misclassification loss, the only property of the score of interest is its sign, and not its magnitude. Therefore, we lose nothing by working with this restriction. Linear loss is therefore a suitable surrogate loss for learning classifiers much like the hinge, logistic, and exponential loss functions. Notice that the linear loss is \emph{not} a convex potential loss. As a surrogate for minimizing $\riskmis(P,f)$ over \emph{all} classifiers $f \in \BoundedF{X}{1}$,  we will minimize $\risklin(S,f)$ over $f \in \Fclass_{\phi}^1$.
\\
\\
For any sample $S \in \allsamples$ define the \emph{mean vector} and \emph{normalized mean vector} as,
\begin{align*} 
\Phi(S) := \frac{1}{n} \sum\limits_{i = 1}^{n} y_i \phi(x_i) \ \text{and}\ & \hat{\Phi}(S) := \frac{\Phi(S)}{\norm{\Phi(S)}}, 
\end{align*}
respectively. Equation \ref{Kernel Mean solution} can be written as $f(x) = \ip{\Phi(S)}{\phi(x)}$. The mean vector arises as the optimal solution for the linear loss.

\begin{lemma}[ \citep{Sriperumbudur2009}]\label{Mean Classifier minimizes Linear Loss} The mean and normalized mean vectors satisfy,
	$$
	\hat{\Phi}(S) = \argmin_{\omega : \norm{\omega} \leq 1} \risklin(S,\omega) = \argmin_{\omega : \norm{\omega} \leq 1} 1 - \langle \omega , \Phi(S) \rangle
	$$
	with minimum linear loss given by $1 - \norm{\Phi(S)}$. Furthermore, classifying using $\ip{\hat{\Phi}(S)}{\phi(x)}$ is equivalent to classifying according to equation \ref{Kernel Mean solution}.
	
\end{lemma}
The proof is a straightforward application of the Cauchy-Schwarz inequality. As $\hat{\Phi}(S) = \lambda \Phi(S)$, $\lambda > 0$, they both produce the same classifier. Changing the norm constraint to $\norm{\omega} \leq r$ merely scales the classifier, and therefore does not change its misclassification performance. Furthermore, we have the following approximation result.

\begin{theorem}[\citep{Altun2006}]\label{Means and Rademacher}
	For all distributions $P$ and for all bounded feature maps $\featuremap{\phi}{X}$,
	$$
	\norm{\Phi(P) - \Phi(S)} \leq \frac{2}{\sqrt{n}} + \sqrt{\frac{2 \log(\frac{1}{\delta})}{n}},
	$$
	with probability at least $1 - \delta$ on a sample $S$ of $n$ independent draws from $P$.	
\end{theorem}
The proof is obtained via a simple application of McDiarmid's inequality. \citep{tolstikhin2016minimax} show that this simple estimate is in fact minimax optimal to estimate $\Phi(P)$. Coupled with the Cauchy-Schwarz inequality, theorem \ref{Means and Rademacher} yields,
$$
\risklin(P,\omega) \leq \risklin(S,\omega) + \frac{2}{\sqrt{n}} + \sqrt{\frac{2 \log(\frac{1}{\delta})}{n}}, \ \forall \omega \ \mathrm{st} \ \norm{\omega} \leq 1.
$$
Therefore the mean classifier minimizes an empirical approximation of $\risklin(P,\omega)$. 

\subsection{Relation to the SVM}

For a regularization parameter $r$, the SVM finds,
$$
\argmin_{\omega : \norm{\omega} \leq r} \frac{1}{n}\sum\limits_{i = 1}^{n} \max(0, 1 - y_i\langle \omega , \phi(x_i) \rangle).
$$
If we take $r = 1$, by Cauchy-Schwarz $\max(0, 1 - y\langle \omega , \phi(x) \rangle) = 1 - y\langle \omega , \phi(x) \rangle$ and the above objective is equivalent to that of theorem \ref{Mean Classifier minimizes Linear Loss}. The mean classifier is the optimal solution to a highly regularized SVM. This has been observed before in \citep{Tibshirani:2002} and \citep{bedo2006efficient}. Proposition 9 of \citep{van2015learning} shows that the mean classifier is the solution obtained from \emph{any} sufficiently regularized method that classifies according to,
$$
\argmin_{\omega : \norm{\omega} \leq r} \frac{1}{n}\sum\limits_{i = 1}^{n} \ell(y_i, \langle \omega , \phi(x_i) \rangle),
$$
with $\ell$ a convex potential loss function.

\subsection{Relation to Maximum Mean Discrepancy}

Let $P,Q \in \PP(X)$ be two distributions defined over the instance space and define \emph{maximum mean discrepancy} \citep{Gretton:2012}, 
\begin{align*}
\mathrm{MMD}_\phi(P,Q) &:= \max_{\omega : \norm{\omega} \leq 1} \frac{1}{2} |\EE_{x \dist P} \ip{\omega}{\phi(x)} - \EE_{x \dist Q} \ip{\omega}{\phi(x)}| = \frac{1}{2} \norm{\Phi(P) - \Phi(Q)}.
\end{align*}
$\mathrm{MMD}_\phi(P,Q)$ can be seen as a restricted variational divergence,
$$
V(P , Q) =\max_{f \in \BoundedF{X}{1}} \frac{1}{2}|\EE_{x \dist P} f(x) - \EE_{x \dist Q}f(x)|,
$$
a commonly used metric on probability distributions, where $f \in \Fclass_\phi^1 \subseteq \BoundedF{X}{1}$. Both variational divergence and MMD are examples of integral probability metrics \citep{muller1997integral}. \citep{Gretton2006} apply empirical approximations of $\Phi(P)$ and $\Phi(Q)$ as a means of \emph{testing} the null hypothesis that $P=Q$. This test can be understood as finding a \emph{classifier} that can distinguish $P$ from $Q$. Here we show that MMD is closely related to classification and linear loss minimization.
\\
\\
Let $P_{\pm} \in \PP(X)$ be the conditional distribution over instances given a positive or negative label respectively. Define the distribution $P \in \PP(X \times Y)$ that first samples $y$ uniformly from $\{-1,1\}$ and then samples $x \dist P_y$. Then,
$$
\mathrm{MMD}_\phi(P_+,P_{-}) = \max_{\omega : \norm{\omega} \leq 1}  |\EE_{(x,y) \dist P} \ip{\omega}{y \phi(x)} | = \norm{\Phi(P)}.
$$
Therefore, if we assume that positive and negative classes are equally likely, the mean classifier classifies using the $\omega$ that ``witnesses" the MMD, i.e. it attains the max in the above.

\subsection{Relation to Kernel Density Estimation}

Obviously, the mean classifier is a \emph{discriminative} approach. Restricting to kernels with $K(x,x') \in [0,1]$ and $\int K(x,x')dx \leq C$, such as the Gaussian kernel, it can be seen as the following \emph{generative} approach: estimate $P$ with $\tilde{P}$, with class conditional distributions estimated by kernel density estimation. Letting $S_{\pm} = \left\{ (x, \pm 1) \right\} \subseteq S$ take,
$$
\tilde{P}( X = x | Y = \pm 1 ) \propto \frac{1}{|S_\pm|}\sum\limits_{x' \in S_\pm} K(x,x'),
$$
and $\tilde{P}(Y = 1) = \frac{|S_+|}{n}$. To classify new instances, use the Bayes optimal classifier for $\tilde{P}$. This yields the same classification rule as (\ref{Kernel Mean solution}). This is the ``potential function rule" discussed in \citep{Devroye1996}.

\subsection{Extension to Multiple Kernels}

To ensure the practical success of any kernel method, it is important that the \emph{correct} feature map be chosen. This is especially true when using the mean classifier. Even for universal $\phi$, it is \emph{not} the case that $\Fclass_\phi^1$ is dense in $\BoundedF{X}{1}$. It is essential therefore that we use the correct feature map.
\\
\\
So far we have only considered the problem of learning with a single feature map, and not the problem of \emph{learning the feature map}. Given $k$ feature maps $\phi_i : X \rightarrow \Hilbert_i$, $i \in [1;k]$, multiple kernel learning \citep{lanckriet2004learning, Bach2008,Hussain2011,Cortes2013} considers learning over a function class that is the convex hull of the classes $\Fclass_{\phi_i}^1$,
$$
\Fclass := \left\{f(x) = \sum_{i=1}^k \alpha_i \ip{\omega_i}{\phi_i(x)}_{\Hilbert_i} : \norm{\omega_i}_{\Hilbert_i} \leq 1, \alpha_i \geq 0, \sum\limits_{i=1}^k \alpha_i = 1\right\}.
$$
Denote the $k$ simplex by $\Delta_k$. By an easy calculation,
\begin{align*}
\min_{f \in \Fclass} \frac{1}{n}\sum\limits_{i = 1}^{n} 1 - y_i f(x_i) &= \min_{\alpha \in \Delta_k, \omega_i \in \Hilbert_i} \sum_{i=1}^k \alpha_i \left(1 - \ip{\omega_i}{\Phi_i(S)}_{\Hilbert_i}\right)  \\
&= \min_{\alpha \in \Delta_k} \sum_{i=1}^k \alpha_i \left(1 - \norm{\Phi_i(S)}_{\Hilbert_i}\right)\\
&= \min_{i \in [1;k]} \left(1 - \norm{\Phi_i(S)}_{\Hilbert_i} \right),
\end{align*}
where the first line follows from the definition of $\Fclass$, the second by minimizing on each $\omega_i$, and the final line follows from the linearity in $\alpha$. In other words, we choose the feature map that minimizes $1 - \norm{\Phi_i(S)}_{\Hilbert_i}$. This is in contrast to the usual multiple kernel learning techniques that generally do not pick out a \emph{single} feature map. Furthermore, we have the following generalization bound.

\begin{theorem}\label{Mean Classifier Bound: Multiple}
	For all distributions $P$ and for all finite collections of bounded feature maps $\phi_i : X \rightarrow \Hilbert_i$, $i \in [1;k]$
	$$
	\risklin(P,\hat{\Phi}_i(S)) \leq \risklin(S,\hat{\Phi}_i(S)) + \frac{2}{\sqrt{n}} + \sqrt{\frac{2\left( \logdelta + \log(k) \right)}{n}}, \forall i \in [1 ;k],
	$$
	with probability at least $1 - \delta$ on a sample $S$ of $n$ independent draws from $P$.
\end{theorem}
The proof proceeds via an application of theorem \ref{Means and Rademacher}, together with a union bound and an application of the Cauchy-Schwarz inequality. Classifying according to the feature map that minimizes $1 - \norm{\Phi_i(S)}_{\Hilbert_i}$ can be understood as minimizing the right-hand side of the bound in Theorem \ref{Mean Classifier Bound: Multiple}.The quantity,
$$
\norm{\Phi(S)} = \sqrt{\frac{1}{n^2} \sum\limits_{i = 1}^{n} \sum\limits_{j = 1}^{n} y_i y_j K(x_i,x_j)},
$$
can be thought of as the ``self-similarity" of the sample, and has appeared previously in the literature in kernels for sets \citep{gartner2002multi}. Our multiple kernel learning approach chooses the kernel with the highest self-similarity, the kernel that on average renders those instances with the same label similar and those with different labels dissimilar.

\section{The Robustness of the Mean Classifier}\label{sec:the-robustness-of-the-mean-classifier}

Invariably, when working with real-world data, one has to deal with training data that has been corrupted in some way. Here, we examine the robustness of the mean classifier to perturbations of $P$. We do not consider the statistical issues of learning from a corrupted distribution. For detailed treatment of such problems, see \citep{JMLR:v18:16-315}. We first show that the degree to which one can approximate a classifier without loss of performance is related to the \emph{margin for error} of the classifier. We then discuss the robustness properties of the mean classifier under the $\sigma$-contamination model of \citep{Huber:1981}. 
\\
\\
The results of section \ref{sec:the-robustness-of-the-mean-classifier} only pertain to \emph{linear} function classes. In the following section, we consider \emph{general} function classes. We show that in this more general setting, linear loss is the \emph{only} loss function that is robust to the effects of symmetric label noise.

\subsection{Approximation Error and Margins}\label{sec:approximation-error-and-margins}

Define \emph{margin loss} at \emph{margin} $\gamma$ to be $\ell_\gamma(y,v) = \pred{y v \leq \gamma}$. Margin loss is an upper bound of misclassification loss. For $\gamma = 0 $, $\ell_{\gamma} = \lmis$. Margin loss is used in place of misclassification loss to produce tighter generalization limits to minimize misclassification loss \citep{bartlett1998sample,shawe1998structural}. For a classifier $f$ to have a small margin loss, it must not just accurately predict the label, it must do so with confidence. Maximizing the margin while forcing $\lmargin{\gamma}(S,\omega) = 0$ is the original motivation for the hard margin SVM \citep{Cortes1995}. Here we relate the margin loss of a classifier $f$ to the amount of slop allowed in approximating $f$.

\begin{theorem}\label{Margins and Approximation} For all distributions $P$ and pairs of classifiers $f, \tilde{f}$ with $\norm{f - \tilde{f}}_{\infty} \leq \epsilon$,
 $$
 \riskmis(P,\tilde{f}) \leq \riskL{\lmargin{\epsilon}}(P,f).
 $$ 	
\end{theorem}
The \emph{margin for error} on a distribution $P$ of a classifier $f$ is given by, 
$$
\Gamma(P,f) := \sup \{\gamma : \lmargin{\gamma}(P,f) = \riskmis(P,f) \}.
$$
For a sample $S$, setting $\epsilon < \Gamma(S,f) $ ensures, 
$$
\riskmis(S,\tilde{f}) \leq \riskL{\lmargin{\epsilon}} (S, f) = \riskmis(S,f),
$$ 
where $\tilde{f}$ is any classifier with $\norm{f - \tilde{f}}_{\infty} \leq \epsilon$. The margin therefore provides means of assessing the degree to which one can approximate a classifier; the larger the margin, the greater the allowed error. 

\subsection{Robustness under \texorpdfstring{$\sigma$}{sigma}-contamination}

Rather than samples from $P$, we assume that the decision maker has access to samples from a perturbed distribution,
$$
\tilde{P} = (1-\sigma) P + \sigma Q, \sigma \in [0,1],
$$
with $Q$ the perturbation or corruption. We can view sampling from $\tilde{P}$ as sampling from $P$ with probability $1-\sigma$ and from $Q$ with probability $\sigma$. It is easy to show that $\Phi(\tilde{P}) = (1-\sigma) \Phi(P) + \sigma \Phi(Q)$. Furthermore,
$$
\norm{\Phi(P) - \Phi(\tilde{P})} = \sigma \norm{\Phi(P) - \Phi(Q)}.
$$
A simple application of the Cauchy-Schwarz inequality yields the following.
\begin{corollary}\label{margin and corruption immunity}
	If $\sigma \norm{\Phi(P) - \Phi(Q)} < \Gamma(P,\Phi(P))$ then $\riskmis(P,\Phi(P)) = \riskmis(P,\Phi(\tilde{P}))$.
	
\end{corollary}
Hence, the margin provides means to assess the immunity of the mean classifier to corruption. Furthermore, as $\norm{\Phi(P) - \Phi(Q)} \leq 2$, if $\sigma < \frac{\Gamma(P,\Phi(P))}{2}$ then the mean classifier is immune to the effects of \emph{any} $Q$. We caution the reader that Corollary \ref{margin and corruption immunity} is a one-way implication. For particular choices of $Q$, one can show greater robustness of the mean classifier. 

\subsection{Learning Under Symmetric Label Noise}\label{SLN learning with linear function classes}

The previous section considered \emph{general} perturbations of $P$. Here we consider one particular perturbation given by symmetric label noise \citep{Angluin1988}. Rather than samples from $P$, the decision maker has access to samples from a corrupted distribution $P_\sigma$. To sample from $P_\sigma$, first draw $(x,y) \dist P$ and then flip the label with probability $\sigma$. Learning from $P_\sigma$ can be understood as a corrupted learning problem of the sort studied by \citep{JMLR:v18:16-315}. This problem is of practical interest, particularly in situations where there are multiple labellers, each of which can be viewed as an ``expert" labeller with added noise. Remarkably, this seemingly benign form of noise can break standard approaches to learning classifiers.
\\
\\
\citep{Long2008} proved the following negative result on what is possible when learning under symmetric label noise: for any $\sigma \in (0, \frac{1}{2})$, there exists a distribution $P$ and a linear function class $\Fclass$ where, when the decision maker observes samples from $P_\sigma$, minimization of \emph{any convex potential} over $\Fclass$ results in classification performance on $P$ which is equivalent to random guessing. The example provided in \citep{Long2008} is far from esoteric, in fact, it is a given by a distribution in $\RR^2$ that is concentrated on three points with function class given by linear hyperplanes through the origin. We review their construction in section \ref{sec:linear-feature-map-and-label-noise}.
\\
\\
The mean classifier avoids these issues. We show that the mean classifier is not affected by symmetric label noise.

\subsubsection{Symmetric Label Noise Immunity of the Mean Classifier}\label{sec:symmetric-label-noise-immunity-of-the-mean-classifier}

In section \ref{sec:the-robustness-of-the-mean-classifier}, one can decompose
$$
P_\sigma = (1-\sigma) P + \sigma P',
$$
where $P'$ is the ``label flipped" version of $P$. It is easy to show $\Phi(P') = -\Phi(P)$. Therefore, $\Phi(P_\sigma) = (1-2\sigma) \Phi(P)$. This simple observation allows us to estimate $\Phi(P)$ from a corrupted sample.

\begin{lemma}\label{noisy mean estimation}
	For all distributions $P$ and for all bounded feature maps $\featuremap{\phi}{X}$,
	$$
	\norm{\Phi(P) - \frac{1}{1 - 2 \sigma}\Phi(S)} \leq \frac{1}{1 - 2 \sigma} \left(\frac{2}{\sqrt{n}} + \sqrt{\frac{2 \log(\frac{1}{\delta})}{n}} \right),
	$$
	with probability at least $1 - \delta$ on a sample $S$ of $n$ independent draws from $P_\sigma$.	
\end{lemma}
The proof is a direct application of theorem \ref{Means and Rademacher}. Coupled with the Cauchy-Schwarz inequality, lemma \ref{noisy mean estimation} yields,
$$
\risklin(P,\omega) \leq 1 - \frac{1}{1 - 2 \sigma}\ip{\Phi(S)}{\omega}+\frac{1}{1 - 2 \sigma} \left(\frac{2}{\sqrt{n}} + \sqrt{\frac{2 \log(\frac{1}{\delta})}{n}} \right), \ \forall \omega \ \mathrm{st} \ \norm{\omega} \leq 1.
$$
The first term in the sum can be interpreted as a correction to the linear loss that takes the noise into account, the second as a penalty term. Notice the extra factor of $\frac{1}{1 - 2 \sigma}$. Theorem \ref{noisy mean estimation} provides an upper bound for minimizing $\risklin(P,\omega)$ from noisy samples. \citep{JMLR:v18:16-315} provides a lower bound of the same form. In short, learning under symmetric label noise is statistically a factor of $\frac{1}{1 - 2 \sigma}$ harder than learning from cleanly labeled data.
\\
\\
Although knowledge of $\sigma$ is required to estimate $\risklin(P,\omega)$, if all we care about is misclassification performance, then, given a large enough training sample, the exact value of $\sigma$ does not matter.

\begin{lemma}\label{something or other}
	For all distributions $P$, bounded feature maps $\phi : X \rightarrow \Hilbert$ and $\sigma \in [0,\frac{1}{2})$,
	$$
	\riskmis(P,\Phi(P)) = \riskmis(P,\Phi(P_\sigma)).
	$$
\end{lemma}
The proof comes from the simple observation that since $\Phi(P)$ and $\Phi(P_\sigma)$ are related by a positive constant, they produce the same classifier. This result extends previous results in \citep{Servedio1999,Kalai2008} on the symmetric label noise immunity of the mean classification algorithm, where it is assumed that the marginal distribution over instances is uniform on the unit sphere in $\RR^n$. 

\subsubsection{Other Approaches to Learning Under Symmetric Label Noise}\label{sec:other-approaches-to-learning-under-symmetric-label-noise}

Ostensibly, \citep{Long2008} establishes that convex losses are not robust to symmetric label noise. This motivates the use of nonconvex losses \citep{Stempfel:2007b, Hamed:2010, Ding2010, Denchev:2012, Manwani:2013}. These approaches are computationally intensive and may scale poorly to large data sets. Furthermore, as demonstrated in the additional material of \citep{van2015learning}, some of these nonconvex losses are not immune to the effects of label noise.
\\
\\
An alternate means of circumventing the impossibility result of \citep{Long2008} is to use a rich function class, say by using a universal kernel \citep{steinwart2001influence,micchelli2006universal}, together with a standard convex potential loss.

\begin{proposition}\label{universal consistency}
	For all distributions $P$ and for all $\sigma \in [0,\frac{1}{2})$, 
	$$
	\argmin_{f \in \BoundedF{X}{1}} \riskmis(P, f) = \argmin_{f \in \BoundedF{X}{1}} \riskmis(P_\sigma, f).
	$$
\end{proposition}
We include a short proof of this proposition in the Appendix. As the Bayes optimal classifier is the same for both noisy and clean data, one can appeal to universality results such as those in \citep{lugosi2004bayes}, and minimize a standard classification-calibrated loss over a large noisy sample and large function class. Although this approach is immune to symmetric label noise, performing the minimization is costly, both statistically and computationally. By Theorem \ref{Surrogate Regret Bound for Linear Loss}, for sufficiently rich function classes, using any of these other losses will produce the same result as using linear loss.	
\\
\\
Finally, if the noise rate is known, one can use the method of unbiased estimators presented by \citep{Natarajan:2013} and correct for corruption. The obvious drawback is that, in general, the noise rate is unknown. In the following section, we explore the relationship between linear loss and the method of unbiased estimators. We show that linear loss is ``unaffected" by this correction (in a sense to be made precise). Furthermore, linear loss is essentially the \emph{only} convex loss with this property.

\subsubsection{Symmetric Label Noise Immunity of Linear Loss Minimization}\label{sec:symmetric-label-noise-immunity-of-linear-loss-minimization}

The weakness of the analysis of Sections \ref{sec:symmetric-label-noise-immunity-of-the-mean-classifier} and \ref{sec:other-approaches-to-learning-under-symmetric-label-noise}, is the focus on \emph{linear} function classes. Here we show that linear loss minimization over \emph{general} function classes is unaffected by symmetric label noise, in the sense that for all $\sigma \in [0,\frac{1}{2})$ and for all function classes $\Fclass \subseteq \RR^X$,
$$
\argmin_{f \in \Fclass} \risklin(P, f) = \argmin_{f \in \Fclass} \risklin(P_\sigma,f).
$$
For the following section we work \emph{directly} with distributions $Q \in \PP(\RR \times Y )$ over score, label pairs. Any distribution $P$ and classifier $f$ induces a distribution $Q(P,f)$ with,
$$
\EE_{(v,y) \dist Q(P,f)} \ell(y,v) = \EE_{(x,y) \dist P} \ell(y,f(x)).
$$
A loss $\ell$ provides means to \emph{order} distributions. For two distributions $Q, Q'$, we say $Q \leq_{\ell} Q'$ if,
$$
\EE_{(v,y) \dist Q} \ell(y, v) \leq \EE_{(v,y) \dist Q'}   \ell(y,v).
$$
If $Q = Q(P,f_1)$ and $Q' = Q(P,f_2)$, the above is equivalent to,
$$
\EE_{(x,y) \dist P} \ell(y, f_1(x)) \leq \EE_{(x,y) \dist P} \ell(y,f_2(x)),
$$
the classifier $f_1$ has lower risk than $f_2$. The decision maker wants to find the distribution $Q$, in some restricted set, that is smallest in the ordering $\leq_{\ell}$. Denote by $Q_\sigma$, the distribution obtained from drawing pairs $(v,y) \dist Q$ and then flipping the label with probability $\sigma$. In light of Long and Servedio's example, there is no guarantee that, 
$$
Q \leq_{\ell} Q' \Leftrightarrow Q_\sigma \leq_{\ell} Q'_\sigma.
$$
In words, noise might affect how distributions are ordered. To progress we seek loss functions that are \emph{robust} to label noise.
\begin{definition}
	A loss $\ell$ is \emph{robust to label noise} if for all distributions $Q,Q'$ and for all $\sigma \in [0,\frac{1}{2})$, 
	$$
	Q \leq_{\ell} Q' \Leftrightarrow Q_{\sigma} \leq_{\ell} Q'_{\sigma}.
	$$
\end{definition}
In words, the decision maker correctly orders distributions if they assume no noise. Robustness to label noise easily implies,
$$
\argmin_{f \in \Fclass} \EE_{(x,y) \dist P} \ell(y,f(x)) = \argmin_{f \in \Fclass} \EE_{(x,y) \dist P_\sigma} \ell(y,f(x)),
$$
for all $\Fclass$. Given any $\sigma \in (0,\frac{1}{2})$,  \citep{Natarajan:2013} showed  how to correct for the corruption by associating with \emph{any} loss a corrected loss, 
$$
\ell_\sigma(y,v) = \frac{(1-\sigma) \ell(y,v) - \sigma \ell(-y,v)}{1- 2 \sigma}.
$$
with the property,
$$
\EE_{(v,y) \dist Q} \ell(y,v) = \EE_{(v,y) \dist Q_\sigma} \ell_\sigma(y,v) ,\ \forall Q \in \PP(\RR \times Y).
$$
This is a specific instance of the corruption-corrected losses considered in \citep{JMLR:v18:16-315}. Robustness to label noise can be characterized by the order equivalence of $\ell$ and $\ell_\sigma$.

\begin{definition}[Order Equivalence]
	
	Two loss functions $\ell_1$ and $\ell_2$ are \emph{order equivalent} if for all distributions $Q, Q' \in \PP(\RR \times Y)$,
	$$
	Q \leq_{\ell_1} Q' \Leftrightarrow Q \leq_{\ell_2} Q'.
	$$
	
\end{definition}
We now characterize the losses that are immune to symmetric label noise.

\begin{theorem}\label{Strong Robustness to Label Noise Lemma}
	$\ell$ is robust to label noise if and only if for all $\sigma \in \left(0,\frac{1}{2}\right)$, $\ell$ and $\ell_\sigma$ are equivalent in order.
\end{theorem}
The decision maker correctly orders distributions if they incorrectly assume noise. Following on from these insights, we now characterize when a loss is robust to label noise.

\begin{theorem}[Characterization of Robustness]\label{Strong robustness to label noise characterization}
	Let $\ell$ be a loss with $\ell(-1,v) \neq \ell(1,v) \ \forall v \in \RR$. Then $\ell$ is robust to label noise if and only if there exists a constant $C$ such that,
	$$
	\ell(1,v) + \ell(-1,v) = C ,\ \forall v \in \RR.
	$$
	
\end{theorem}
\citep{Ghosh2015} prove the forward implication. Misclassification loss satisfies the conditions for theorem \ref{Strong robustness to label noise characterization}, however it is difficult to minimize directly. For linear loss,
$$
\ell(1,v) + \ell(-1,v) = 1 - v + 1 + v = 2.
$$
Therefore linear loss is robust to label noise. Furthermore, up to order equivalence, linear loss is the only convex function that satisfies \ref{Strong robustness to label noise characterization}.

\begin{theorem}[Uniqueness of Linear Loss]\label{Uniqueness of Linear Loss}
	A loss $\ell$ is convex in its second argument and is robust to label noise if and only if there exists a constant $\lambda$ and a function $g : Y \rightarrow \RR$ such that, 
	$$
	\ell(y,v) = \lambda y v + g(y).
	$$
	Furthermore $\ell$ is classification calibrated if and only if $\lambda < 0$.

\end{theorem}

\subsubsection{Beyond Symmetric Label Noise}\label{sec:beyond-symmetric-label-noise}

Thus far we have assumed that the noise on positive and negative labels is the same. A sensible generalization is label conditional noise, where the label $y \in \{-1,1\}$ is flipped with a label-dependent probability $\sigma_{\pm}$. Following \citep{Natarajan:2013}, we can correct for class conditional label noise and use the loss,
$$
\ell_{\sigma_{-}, \sigma_+}(y,v) = \frac{(1-\sigma_{-y}) \ell(y,v) - \sigma_y \ell(-y,v)}{1- \sigma_{-1} - \sigma_1}.
$$

\begin{theorem}\label{No loss is immune to class conditional label noise}
	Let $\sigma_{-} + \sigma_{+} < 1$ and $\ell$ be a loss with $\sigma_{+} \ell(-1,v) + \sigma_{-} \ell(1,v) = C$ for all $v\in \RR$, for some constant $C$. Then $\ell_{\sigma_{-}, \sigma_+} $ and $\ell$ are equivalent in order.	
\end{theorem}
Therefore, if the decision maker knows the ratio $\frac{\sigma_{-1}}{\sigma_{1}}$, then for a certain class of loss functions they can avoid estimating noise rates. For linear loss,
$$
\sigma_{+} (1 + v) + \sigma_{-} (1-v) = \sigma_+ + \sigma_{-} + (\sigma_+ - \sigma_{-}) v,
$$
which is not constant in $v$ unless $\sigma_+ = \sigma_{-}$. Linear (and similarly misclassification loss) are no longer robust under label conditional noise. This result also means there is no non trivial convex loss that is robust to label conditional noise for all noise rates $\sigma_{-} + \sigma_+ < 1$, as linear loss would be a candidate for such a loss.
\\
\\
Progress can be made if one works with more general error measures, beyond expected loss. For a distribution $P \in \PP(X\times Y)$, let $P_+, P_- \in \PP(X)$ be the conditional distribution over instances given a positive or negative label respectively. The balanced error function is defined as,
$$
\mathrm{BER}_{\ell}(P_+,P_-,f) := \frac{1}{2} \EE_{x\dist P_+} \ell(1,f(x)) + \frac{1}{2} \EE_{x \dist P_-} \ell(-1,f(x)).
$$
If both labels are equally likely under $P$, then the balanced error is exactly the expected loss. The balanced error ``balances" the two class, treating errors on positive and negative labels equally. Closely related to the problem of learning under label conditional noise, is the problem of learning under mutually contaminated distributions \citep{scott2013classification,Menon2015}. Rather than samples from the clean label conditional distributions, the decision maker has access to samples from corrupted distributions $\tilde{P}_{+}, \tilde{P}_-$, 
\begin{align*} 
\tilde{P}_+ = (1 - \alpha) P_+ + \alpha P_-  \ \text{and}\ & \tilde{P}_- = \beta P_+ + (1 - \beta) P_- ,\ \alpha + \beta < 1.
\end{align*}
In words, the corrupted $\tilde{P}_{y}$ is a combination of the true $P_y$ and the unwanted $P_{-y}$. We warn the reader that $\alpha$ and $\beta$ are \emph{not} the noise rates on the two classes. However, in section 2.3 of \citep{Menon2015}, they are shown to be related to $\sigma_{\pm}$ by an invertible transformation.

\begin{theorem}\label{Balanced Error Noise Immunity}
	Let $\ell$ be robust to label noise. Then,
	$$
	\mathrm{BER}_{\ell}(\tilde{P}_+,\tilde{P}_-,f) = (1 - \alpha - \beta) \mathrm{BER}_{\ell}(P_+,P_-,f) + \frac{\left(\alpha + \beta \right)}{2} C,
	$$
	for some constant $C$.
\end{theorem}
This is a generalization of proposition 1 of \citep{Menon2015}, which is restricted to misclassification loss. Taking argmins yields,
$$
\argmin_{f \in \Fclass} \mathrm{BER}_{\ell}(\tilde{P}_+,\tilde{P}_-,f) = \argmin_{f \in \Fclass} \mathrm{BER}_{\ell}(P_+,P_-,f).
$$
Thus balanced error can be optimized from corrupted distributions. Observe that this result holds for \emph{any} function class $\Fclass$
\\
\\
Corollary 3 of  \citep{Menon2015} shows that the AUC is also unaffected by label conditional noise.
\\
\\
Going further beyond symmetric label noise, one can assume a general noise process with noise rates that depend both on the label and the observed instance. Define the noise function $\sigma : X \times Y \rightarrow [0,\frac{1}{2})$, with $\sigma(x,y)$ the probability that the instance label pair $(x,y)$ has its label flipped. Rather than samples from $P$, the decision maker has samples from $P_\sigma$, where to sample from $P_\sigma$ first sample $(x,y) \dist P$ and then flip the label with probability $\sigma(x,y)$. The recent work of \citep{Ghosh2015} proves the following theorem concerning the robustness properties of minimizing any loss that is robust to label noise.

\begin{lemma}\label{robustness under all noise: theorem}
	For all distributions $P$, function classes $\Fclass$, noise functions $\sigma : X \times Y \rightarrow [0,\frac{1}{2})$ and loss functions $\ell$ that are robust to label noise,
	$$
	\riskL{\ell}(P,f_\sigma^*) \leq \frac{\riskL{\ell}(P,f^*)}{1 - 2 \max_{(x,y)} \sigma(x,y)},
	$$
	where $f_\sigma^*$ and $f^*$ are the minimizers over $\Fclass$ of $\riskL{\ell}(P_\sigma,f)$ and $\riskL{\ell}(P,f)$ respectively.
	
\end{lemma}
This is a slight generalization of remark 1 in \citep{Ghosh2015}. There, they only consider variable noise rates that are functions of the instance. We include it for completeness. In particular, this theorem shows that if $\riskL{\ell}(P,f^*) = 0$ and, 
$$
\max_{(x,y)} \sigma(x,y) < \frac{1}{2},$$
then minimizing $\ell$ with samples from $P_\sigma$ will also recover a classifier with $\riskL{\ell}(P,f^*) = 0$.

\section{Sparse Approximation of Kernel Classifiers}\label{Sparse Approximation}

The main problem of classifying according to equation \ref{Kernel Mean solution} is the dependence of the classifier on the \emph{entire} sample. If the sample is large, the mean classifier will take a long time to evaluate. We now show how this can be alleviated. 
\\
\\
For this section, the sample will be an arbitrary finite subset $S = \{\omega_i\}_{i = 1}^{n} \subseteq \Hilbert$. The previous setting can be recovered by taking $\omega_i = y_i \phi(x_i)$. Denote by,
$$
\chull(S) = \left\{\sum_{\omega \in S} \alpha(\omega) \omega : \alpha \in \RR_+^S, \sum_{\omega \in S} \alpha(\omega) = 1\right\},
$$
the convex hull of $S$. Elements of $\chull(S)$ can be thought of as weighted sub-samples of $S$, with weights specified by the probability distribution $\alpha$. For a subset $S' \subseteq S$, define,
$$
\alpha(S') = \sum_{\omega \in S'} \alpha(\omega).
$$
We say $\omega^* \in \chull(S)$ is $k$-sparse if its corresponding weight function $\alpha^*$ has only $k$ non-zero entries. We consider the problem of approximating $\omega^* \in \chull(S)$ with a $k$-sparse $\tilde{\omega} \in \chull(S)$. In the context of kernel classifiers, $\omega^*$ is the output of a learning algorithm such as equation \ref{Kernel Mean solution}. By Cauchy-Schwarz, controlling $\norm{\omega^* - \tilde{\omega}}$ directly controls the distance between their respective classifiers. A naive method to obtain a sparse approximation is to use the mean of a random sample from $\alpha$. Via an application of theorem \ref{Means and Rademacher} such a scheme guarantees,
$$
\norm{\omega^* - \tilde{\omega}} \leq O\left(\frac{1}{\sqrt{k}}\right),
$$
with high probability. We first present a lower bound that shows that this is the best one can hope to do in general. We then demonstrate how a simple refinement to random subsampling leads to a method that adapts to the complexity of the sample.

\subsection{A Lower Bound for Sparse Approximation}

We remind the reader that kernel-based methods proceed via mapping the instances into a Hilbert space of high, or even infinite dimension. It is precisely in the infinite-dimensional setting where one cannot beat random sub-sampling.

\begin{theorem}\label{Sparse Approximation Lower Bound}
	
	Let $\Hilbert$ be a separable Hilbert space of infinite dimension. For all $n > 0$ there exists a sample $S \subseteq \Hilbert$ of size $n$ and a $\omega^* \in \chull(S)$ such that for all $k$-sparse $\tilde{\omega} \in \chull(S)$, 
	$$
	\norm{\omega^* - \tilde{\omega}} \geq \sqrt{\frac{1}{k} - \frac{1}{n}}.
	$$
	
\end{theorem}
Taking a sufficiently large sample yields a lower bound of order $\frac{1}{\sqrt{k}}$. The sample that yields this lower bound has $\ip{\omega_i}{\omega_j} = 0$ if $i \neq j$. This sample is incompressible as no two instances are similar.

\subsection{Sparse Approximation via the Exploitation of Clusters}\label{sec:sparse-approximation-via-the-exploitation-of-clusters}

While theorem \ref{Sparse Approximation Lower Bound} shows that in general one cannot hope to outperform random sub-sampling, for specific samples $S$ one can do much better.  It can be the case that $S$ ``clusters" more in certain regions of $\Hilbert$. Random subsampling does not exploit this. Here we show how a more refined scheme can be used to give stronger approximation guarantees. 

\begin{theorem}[Clustered Sub-Sampling]\label{Binning and sub-sampling for fun and profit}
	Let $S$ be a finite subset of a Hilbert space $\Hilbert$ and $S_i \subseteq S$, $i \in [1;m]$ be a partition of $S$ with diameter, 
	$$
	D = \sup_{i \in [1;m]} \sup_{\omega, \omega' \in S_i} \norm{\omega - \omega'}.
	$$ 
	Furthermore, let $\omega^* \in \chull(S)$ with corresponding weight function $\alpha^*$. Construct the approximation $\tilde{\omega} \in \chull(S)$ as follows:
	
	\begin{enumerate}
		\item For $i \in [1;m]$, sample $n_i = \lceil \alpha^*(S_i) m \rceil$ elements $\omega_j \in S_i$ with probability proportional to $\alpha^*(\omega_j)$, and set $\tilde{\omega}_i = \frac{1}{n_i} \sum_{j=1}^{n_i} \omega_j$.
		\item Set $\tilde{\omega} = \sum_{i=1}^{m} \alpha^*(S_i) \tilde{\omega}_i$.
	\end{enumerate}
	Then $\tilde{\omega}$ is at most $2m$-sparse. Furthermore, with probability at least $1-\delta$,
	$$
	\norm{\omega^* - \tilde{\omega}} \leq D \left(\frac{1}{\sqrt{m}} + \sqrt{ \frac{ \logdelta}{m}}\right).
	$$
	
\end{theorem}
Theorem \ref{Binning and sub-sampling for fun and profit} states that to construct an accurate $2m$-sparse approximation to $\omega^* \in \chull(S)$, it suffices to find a partition of $S$ with $m$ elements that has a small diameter. Assuming that the partition has already been calculated, the clustered subsampling runs in time order $m n$.
\\
\\
We denote the minimum diameter of any $m$ set partition of $S$ by $D^*(S,m)$. Although in general calculating the \emph{optimal} partition is NP hard, a simple greedy algorithm can be used to produce a diameter partition at most twice that of the optimal \citep{gonzalez1985clustering}. Coupled with the sampling scheme of \ref{Binning and sub-sampling for fun and profit}, this algorithm provides a means to approximate sparsely $\omega^* \in \chull(S)$. The pseudocode for this approach is Algorithm 1. 
\\
\\
Naively, algorithm 1 runs in time order $m^2 n$, but it can be implemented to run in time order $m n$. This is because when adding a new point to $\tilde{S}$, one only needs to calculate distances to the most recently added point to $\tilde{S}$ (this runs in order $n$ time). Together with the sampling scheme of theorem \ref{Binning and sub-sampling for fun and profit}, algorithm 1 provides simple means to approximate sparsely $\omega^* \in \chull(S)$ that runs in time order $m n$. The parameter $m$ in Algorithm 1 controls the sparsity of $\tilde{\omega}$. Alternately, through a slight modification to Algorithm 1, a target error tolerance can be established $\epsilon$. The pseudocode for this approach is Algorithm 2.

\begin{algorithm}
	\KwIn{Sample $S = \{\omega_i\}_{i=1}^{n} \subseteq \Hilbert$, target $\omega^* \in \chull(S)$, maximum number of partitions $m$ and failure probability $\delta$.} 
	\KwResult{$\tilde{\omega} \in \chull(S)$ that is at most $2m$-sparse with, $\norm{\omega^* - \tilde{\omega}} \leq 2 D^*(S,m) \left(\frac{1}{\sqrt{m}} + \sqrt{ \frac{ \logdelta}{m}}\right)$, with probability at least $1 - \delta$.
	} 
	
	\textbf{Initialization}: Choose $\omega_1 \in S$ arbitrarily and let $\tilde{S} = \{\omega_1\}$\;
	\While{$|\tilde{S}| \leq m$}{
		Let $\omega^* = \argmax_{\omega \in S} \min_{\tilde{\omega} \in \tilde{S}} \norm{\omega - \tilde{\omega}}$\;
		Add $\omega^*$ to $\tilde{S}$.
	}
	\textbf{Then}: Partition $S$ according to the closest element of $\tilde{S}$, $S_i$ comprises all elements in $S$ that are closest to $\tilde{\omega}_i \in \tilde{S}$. \;
	\KwOut{$\tilde{\omega}$ obtained from clustered-subsampling using the above partition.}
	
	\caption{Farthest First Traversal.}
\end{algorithm}

\begin{algorithm}
	\KwIn{Sample $S = \{\omega_i\}_{i=1}^{n} \subseteq \Hilbert$, target $\omega^* \in \chull(S)$, maximum number of partitions $m$ and failure probability $\delta$.} 
	\KwResult{Potentially sparse $\tilde{\omega} \in \chull(S)$ with, $\norm{\omega^* - \tilde{\omega}} \leq \epsilon$ with probability at least $1 - \delta$.
	} 
	
	\textbf{Initialization}: Choose $\omega_1 \in S$ arbitrarily and let $\tilde{S} = \{\omega_1\}$\;
	\While{$2d \left(\frac{1}{\sqrt{k}} + \sqrt{ \frac{ \logdelta}{k}}\right) > \epsilon$}{
		Let $\omega^* = \argmax_{\omega \in S} \min_{\tilde{\omega} \in \tilde{S}} \norm{\omega - \tilde{\omega}}$\;
		$d \leftarrow  \max_{\omega \in S} \min_{\tilde{\omega} \in \tilde{S}} \norm{\omega - \tilde{\omega}}$\;
		$k \leftarrow k + 1$\;
		Add $\omega^*$ to $\tilde{S}$.
	}
	\textbf{Then}: Partition $S$ according to the closest element of $\tilde{S}$, ie $S_i$ comprises all elements in $S$ that are closest to $\tilde{\omega}_i \in \tilde{S}$. \;
	\KwOut{$\tilde{\omega}$ obtained from clustered-subsampling using the above partition.}
	\caption{Modified Farthest First Traversal.}
\end{algorithm}

\newpage

\subsection{Approximating Elements in the Span of the Sample}

We have considered approximating elements in the convex hull of the sample. For general kernels methods, it is often the case the optimal $\omega^*$ is in the \emph{span} of the sample. Here we show how to use clustered sub-sampling to approximate $\omega^* \in \Span(S)$. Denote by,
$$
\Span(S) := \left\{\sum_{\omega \in S} \alpha(\omega) \omega : \alpha \in \RR^S\right\}.
$$
the \emph{span} of $S$. Let $\omega^* \in \Span(S)$. Then,
\begin{align*}
\omega^* &= \sum_{\omega \in S} \alpha^*(\omega) \omega \\ 
&= \sum_{\omega \in S} |\alpha^*(\omega)| \sign(\alpha^*(\omega)) \omega \\ 
&= \underbrace{\left( \sum_{\omega \in S} |\alpha^*(\omega)| \right)}_{\text{total weight}} \underbrace{ \left(\sum_{\omega \in S} \frac{|\alpha^*(\omega)|}{\sum_{\omega \in S} |\alpha^*(\omega)|} \sign(\alpha^*(\omega)) \omega \right)}_{\pi^* \in \chull(\sign_{\alpha^*} (S))},
\end{align*}
where the first term can be understood as the total weight of $\omega^*$, and the second term, $\pi^*$, an element in the convex hull of the \emph{signed} sample,
$$
\sign_{\alpha^*} (S) := \{\sign(\alpha^*(\omega)) \omega : \omega \in S\}.
$$
To approximate $\omega^* \in \Span(S)$, we first write $\omega^* = \left(\sum_{\omega \in S} |\alpha^*(\omega)| \right) \pi^*$, we then approximate $\pi^*$ with $\tilde{\pi} \in \chull(\sign_{\alpha^*} (S))$ via clustered subsampling. Finally we take, 
$$
\tilde{\omega} = \left(\sum_{\omega \in S} |\alpha^*(\omega)| \right) \tilde{\pi}.
$$

\subsection{Parallel Extension}

In Theorem \ref{Binning and sub-sampling for fun and profit} we made use of a partition of $S$ to produce a sparse approximation of $\omega^* \in \chull(S)$. Partitions can also be used to parallelize any procedure for constructing sparse approximations. One has,
$$
\sum_{\omega \in S} \alpha(\omega) \omega = \sum_{i=1}^k \alpha(S_i) \left( \sum_{\omega \in S_i} \frac{\alpha(\omega)}{\alpha(S_i)}\omega \right),
$$
where we have split an average over $S$ into $k$ averages over the disjoint subsets $S_i$, $i\in [1;k]$. If we approximate each sub-average to tolerance $\epsilon$, combining the approximations yields an approximation to the total average with tolerance $\epsilon$.

\begin{lemma}[Parallel Means]\label{Parallel Means}
	Let $\omega = \sum \lambda_i \omega_i$ with $\lambda_i \geq 0$ and $\sum \lambda_i = 1$. Suppose that for each $i$ there is an approximation $\tilde{\omega}_i$ with $\norm{\omega_i - \tilde{\omega}_i} \leq \epsilon$. Then $\norm{\omega - \sum \lambda_i \tilde{\omega}_i} \leq \epsilon$.
\end{lemma}
The proof is a simple application of the triangle inequality and the homogeneity of norms. The lemma \ref{Parallel Means} allows one to use a map reduction algorithm to sparsely represent large data sets. The data is split into $K$ groups and then sparsely approximates the mean of each group.
\\
\\
The cost of parallelization is a possibly denser approximation, as the following example shows. Consider the following sample $S = \{1, 1, 0, 0\}$, that is, $S$ consists of two duplicates of $1$ and $0$. Using the standard linear kernel, $D^*(S, 2) = 0$, $S$ can be perfectly approximated by two elements. However, naively partitioning $S$ into two sets $S_i = \{0,1\}$ each with one copy of $0$ and $1$ also has $D^*(S_i, 2) = 0$. Combining the sparse approximations of $S_i$ yields the approximation to $S$ with four elements. 
\\
\\
This issue can be alleviated by a second round of sparse approximation. 

\subsection{Comparisons with Previous Work}

\subsubsection{Algorithmic Luckiness}

The subsampling scheme presented in Theorem \ref{Binning and sub-sampling for fun and profit} appeared previously in the appendix of \citep{Herbrich2003}. There it was used to establish the existence of a $k$-sparse approximation $\tilde{\omega}$ with,
$$
\norm{\omega^* - \tilde{\omega}} \leq \frac{\sqrt{2} D^*(S,\frac{k}{2})}{\sqrt{k}}.
$$ 
They did not provide a computationally feasible means of constructing a near-optimal partition nor provided a concentration result. Theorem \ref{Binning and sub-sampling for fun and profit} coupled with algorithm 1 provides a computationally feasible scheme for constructing a $k$-sparse $\tilde{\omega}$ with,
$$
\norm{\omega^* - \tilde{\omega}} \leq 2 \sqrt{2} D^*\left(S,\frac{k}{2}\right) \left(\frac{1}{\sqrt{k}} + \sqrt{ \frac{ \logdelta}{k}}\right),
$$
with probability at least $1 - \delta$.

\subsubsection{Kernel Herding}

An alternate approach to random sampling is to directly attack the following optimization problem,
$$
\min_{\tilde{\omega} \in \chull(S)} \norm{\omega^* - \tilde{\omega}}^2.
$$
By utilizing a greedy optimization algorithm, a sparse approximation can be obtained. Kernel herding \citep{Welling2009,Chen2010} is one such approach. In general herding gives the same approximation guarantees as random sampling. There has been much interest in when herding gives \emph{faster} rates of convergence. Proposition 1 of \citep{Chen2010} demonstrates how a simple greedy procedure yields,
$$
\norm{\omega^* - \tilde{\omega}} \leq O(\frac{1}{d k}),
$$ 
where $d$ is the distance of $\omega^*$ to the boundary of $\chull(S)$. This scheme has the same computational complexity as ours. \citep{Bach2012} showed an equivalence between herding procedures and the Frank-Wolfe method for solving convex problems \citep{wolfe1976finding}. Via this correspondence, they produced more complicated algorithms, with equal or greater computational complexity, than that of \citep{Chen2010} with the apparently better rate of convergence,
$$
\norm{\omega^* - \tilde{\omega}} \leq O(e^{-d k}).
$$
We remark here that while these methods \emph{appear} to give better rates of convergence than our simple sampling scheme, in reality the constant $d$ is so small that this is not the case, as theorem \ref{Sparse Approximation Lower Bound} confirms.
\\
\\
Although the empirical performance of herding algorithms is impressive, at present there is no proof that these methods adapt to the complexity of the sample.

\subsubsection{Sparse Approximation of a Kernel Mean}

\citep{cortes2015sparse} also consider the problem of sparsely approximating a kernel mean. They also utilize farthest first traversal to construct a set of representative points $\tilde{S} \subseteq S$, but rather than clustering and then sub-sampling, they project onto the span of $\tilde{S}$. Their method guarantees,
$$
\norm{\tilde{\omega} - \omega^*} \leq \left(1 - \frac{m}{n}\right) D^*(S, m),
$$ 
with $\tilde{\omega}$ $m$-sparse.

\subsubsection{Sparsity Inducing Objectives versus Sparsity Inducing Algorithms}

Much of practical machine learning can be understood as solving regularized sample risk problems, 
$$
\min_{\omega \in \Hilbert} \frac{1}{n}\sum\limits_{i = 1}^{n} \ell(y_i, \langle \omega , \phi(x_i)) + \Omega(\omega),
$$
with $\ell$ a loss and $\Omega$ a regularizer. It is desirable for the evaluation speed of the outputted classifier that $\omega$ be as sparse as possible. For example, the linear loss objective does not return a sparse solution.
\\
\\
One can understand objectives that promote sparsity, via sparsity inducing losses or sparsity inducing regularizers. For example, in the Lasso, the L1 regularizer $\Omega(\omega) = \lambda \sum_{i=1}^{n} \left| \omega_i \right|$ is used \citep{Tibshirani1996}. Alternately, \citep{Bartlett2007} use the standard square norm regularizer $\Omega(\omega) = \frac{\lambda}{2} \norm{\omega}^2$, and vary the loss. They show there is an inherit trade off between sparse solutions, and solutions that give calibrated probability estimates. Note that this is for a \emph{particular} choice of regularizer. In this approach, the properties of the \emph{actual} minimizer are deduced from the KKT conditions of the relevant optimization objective. 
\\
\\
In practice, one rarely returns the \emph{exact} minimizer. Therefore, the search for \emph{objectives} that have sparse minimizers does not tell the full story. The approach taken in Section \ref{Sparse Approximation} is to find a single method that can be used to sparsely approximate any $\omega \in \chull(S)$, be it the optimal $\omega$ for one of the objectives above, or be it a $\omega$ that is generated via some other scheme.

\section{Tying it All Together: The Robustness, Sparsity Trade-off}

Recall in Section \ref{sec:approximation-error-and-margins} that the margin of error of a classifier measures the degree to which it can be approximated without an increase in its misclassification risk. This ``budget" can be spent on a variety of different approximations, be it a finite sample, noise on the labels, or the sparsity of the final classifier. We can understand this trade-off through a combination of our previous results.
\begin{corollary}\label{stick it together}
For all distributions $P$, $\sigma \in [0, \frac{1}{2})$ and $m > 0$, 
$$
\hnorm{\Phi(P) - \tilde{\omega}} \leq \frac{1}{1 - 2 \sigma} \left(\frac{2}{\sqrt{n}} + \sqrt{\frac{2 \log(\frac{2}{\delta})}{n}} \right) + \frac{2 D^*(S, m)}{1 - 2 \sigma} \left(\frac{1}{\sqrt{m}} + \sqrt{\frac{\log(\frac{2}{\delta})}{m}}\right),
$$
with probability at least $1 - \delta$, where $\tilde{\omega}$ is the output of algorithm 1 on a sample $S$ comprising of $n$ independent draws from $P_\sigma$. Furthermore, $\tilde{\omega}$ is at most $2m$-sparse.

\end{corollary}
The proof proceeds via a combination of lemma \ref{noisy mean estimation}, the approximation guarantee of algorithm 1, the triangle inequality, and finally a union bound. The first term on the right-hand side of the bound can be interpreted as the purely statistical penalty of approximating $\Phi(P)$ via a finite sample, with possible noise on the labels. The second term shows an interesting interaction between sparse approximations and label noise.
\\
\\
First, label noise directly affects the quality of the sparse approximation by a factor of $\frac{1}{1 - 2 \sigma}$. Second, and perhaps more subtlety, for $\sigma > 0$ it can be the case that noisy samples have a larger diameter than clean examples.
\\
\\
Consider the example of figure \ref{fig:checkerboard} of section \ref{sec:sparse-approximation}. Although $D^*(S, 16)$ is small, injecting a small amount of noise into the labels increases the diameter. This is because while the clean sample $S$ has $16$ clusters, a noisy sample will potentially have $32$ clusters.  To get a ``noise-free" perspective of the problem, one can upper bound the diameter of $S$ with the diameter of
$$
\pm S : = {\pm \omega: \omega \in S}.
$$
As $S \subset \pm S$, $D^*(S,m) \leq D^*(\pm S, m)$. Furthermore, $D^*(\pm S, m)$ is not affected by potential noise on the labels.

\section{Experiments}\label{sec:experimental-validation}

Here we provide experimental corroboration of our results. We begin by illustrating the power of clustered subsampling as a means to sparsely approximate kernel expansions. We give an example showing when clustered sub-sampling out performs random sub-sampling. We then illustrate the robustness properties of the mean classifier in the example of \citep{Long2008} and on several UCI data sets. 

\subsection{Sparse Approximation}\label{sec:sparse-approximation}

Figure \ref{fig:checkerboard} illustrates a binary classification problem in which the instances of each class clearly form clusters. One can see that there are 16 clusters, half of which comprise of positively labeled instances, the other negatively labeled. We utilize a Gaussian kernel with kernel function and distance given by,
\begin{align*} 
K(x,x') = \exp\left(-\frac{\norm{x - x'}_2^2}{2 \kappa^2}\right) &\ \text{and}\ \norm{y \phi(x) - y' \phi(x')} = \sqrt{2 - 2 y y' \exp\left(-\frac{\norm{x - x'}_2^2}{2 \kappa^2}\right)},
\end{align*}
with the suitably chosen $\kappa$. Note that any two instances with \emph{different} labels are at least $\sqrt{2}$ apart. Figure \ref{fig:checkerboardClustered} was produced by the farthest first traversal of the sample from Figure \ref{fig:checkerboard} for $m=16$ iterations, before clustering and then sub-sampling, yielding an approximation to the mean of sparsity $k=32$. The sparse classifier obtained from Figure \ref{fig:checkerboardClustered} correctly classifies all instances in figure \ref{fig:checkerboard}. In contrast, randomly sampling $32$ elements from the data set of figure \ref{fig:checkerboard} will with high probability miss one of the $16$ clusters, producing an inferior classifier.

\begin{minipage}{\textwidth}
	\begin{minipage}[c]{0.45\textwidth}
		
		\includegraphics[width=\linewidth]{checkerboard}
		\captionof{figure}{Checkerboard data set, illustrating the utility of clustered sub-sampling. See text.}
		\label{fig:checkerboard}
		
	\end{minipage}	
	\quad
	\begin{minipage}[c]{0.45\textwidth}
		
		\includegraphics[width=\linewidth]{checkerboardSparse}
		\captionof{figure}{Sparse Approximation of the checkerboard data set. See text.}
		\label{fig:checkerboardClustered}
	\end{minipage}
\end{minipage}

\subsection{Robustness Guarantees}\label{sec:linear-feature-map-and-label-noise}

We first show that the linear risk minimizer performs well in the example of \citep{Long2008}. Figure \ref{fig:long2} shows the distribution $P$, where $X = \{ ( 1, -1 ), ( 1, 3 ), (30, 0 ) \} \subset \RR^{2}$, with instances chosen with probability $\frac{1}{2}, \frac{1}{4}$ and $\frac{1}{4}$, respectively. All three instances are labeled positive. We use the identity feature map, with the corresponding linear function class,
$$
\Fclass = \{f(x) = \omega_1 x_1 + \omega_2 x_2 : \omega_1, \omega_2 \in \RR \}.
$$
Solving for,
$$
\argmin_{f\in \Fclass} \riskL{\mathrm{hinge}}(P,f)  = \argmin_{\omega \in \RR^2}\EE_{(x,y) \dist P} \max(0,1-\ip{\omega}{x}),
$$
yields the solid black hyperplane, which correctly classifies all points. Solving for, 
$$
\argmin_{f\in \Fclass} \riskL{\mathrm{hinge}}(P_\sigma,f),
$$
for $\sigma = 0.15$, yields the dashed black hyperplane, which incorrectly classifies the southern most point. As this point is chosen with probability $\frac{1}{2}$, this classifier performs as well as random guessing. The scale of the data set can be chosen so that this occurs for $\sigma$ arbitrarily small. 
\\
\\
In figure \ref{fig:long2}, we show the performance of the mean classifier in the Long and Servedio data set. In contrast to the SVM, the mean classifier provides the red hyperplane, which correctly classifies all data points, for all $\sigma \in [0, \frac{1}{2})$.
\\
\\
We next consider empirical risk minimizers from a random training sample: we construct a training set of $800$ instances drawn from $P_\sigma$. We evaluated the classification performance on a test set of $1000$ instances drawn from $P$. We repeat the experiment for various noise rates. We compare the hinge, linear, and the $t$-logistic loss functions (for $t = 2$) \citep{Ding2010}. From Table \ref{tbl:long-matlab}, even when $\sigma = 0.4$, the unhinged classifier is able to find a perfect solution. In contrast, both other losses suffer at even moderate noise rates.
\\
\\

\begin{minipage}{\textwidth}
	\begin{minipage}[c]{0.45\textwidth}
		
		\includegraphics[width=0.9\linewidth]{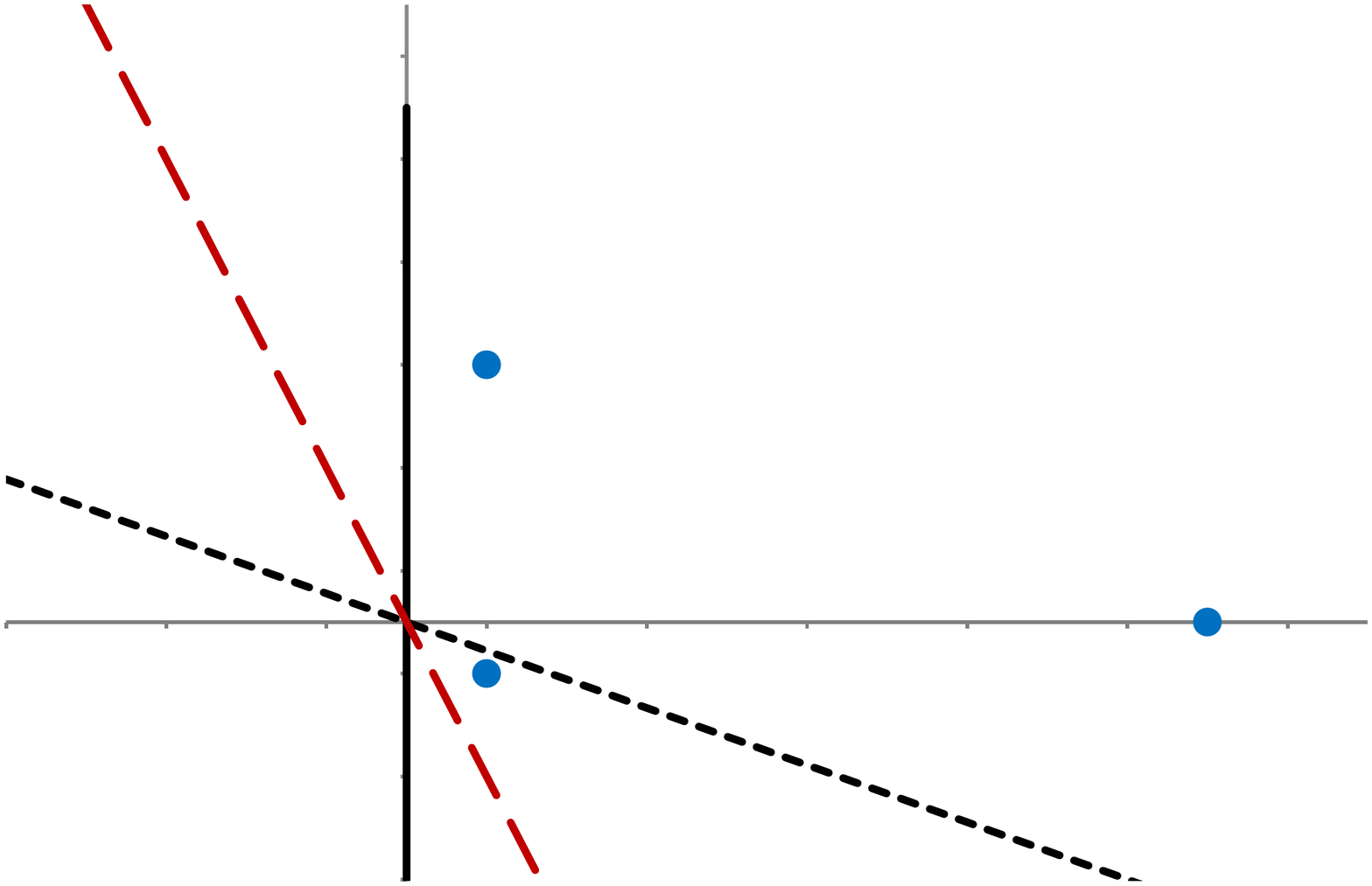}
		\captionof{figure}{Mean classifier performance on Long and Servedio data set.}
		\label{fig:long2}
		
	\end{minipage}	
	\quad
	\begin{minipage}[c]{0.45\textwidth}
		{\scriptsize
			\begin{tabular}{@{}llll@{}}
				\toprule
				\toprule
				& \textbf{Hinge} & \textbf{$t$-logistic} & \textbf{Linear} \\ 
				\midrule
				$\sigma = 0$ & \cellcolor{gray!25}{0.00 $\pm$ 0.00} & \cellcolor{gray!25}{0.00 $\pm$ 0.00} & \cellcolor{gray!25}{0.00 $\pm$ 0.00} \\ 
				$\sigma = 0.1$ & 0.15 $\pm$ 0.27 & \cellcolor{gray!25}{0.00 $\pm$ 0.00} & \cellcolor{gray!25}{0.00 $\pm$ 0.00} \\ 
				$\sigma = 0.2$ & 0.21 $\pm$ 0.30 & \cellcolor{gray!25}{0.00 $\pm$ 0.00} & \cellcolor{gray!25}{0.00 $\pm$ 0.00} \\ 
				$\sigma = 0.3$ & 0.38 $\pm$ 0.37 & 0.22 $\pm$ 0.08 & \cellcolor{gray!25}{0.00 $\pm$ 0.00} \\ 
				$\sigma = 0.4$ & 0.42 $\pm$ 0.36 & 0.22 $\pm$ 0.08 & \cellcolor{gray!25}{0.00 $\pm$ 0.00} \\ 
				$\sigma = 0.49$ & 0.47 $\pm$ 0.38 & 0.39 $\pm$ 0.23 & \cellcolor{gray!25}{0.34 $\pm$ 0.48} \\ 
				\midrule
			\end{tabular}
		}
		\captionof{table}{Mean and standard deviation of the 01 risk over 125 trials.}
		\label{tbl:long-matlab}
	\end{minipage}
\end{minipage}

\section{Conclusion}

It is well known that no single learning algorithm is best in all circumstances. We have studied the mean classifier and demonstrated its robustness to various types of noise and shown that its apparent deficiency (lack of sparseness of the solution) can be substantially alleviated with a tractable sparsification algorithm. The result is a conceptually clear and theoretically justified means of learning classifiers.

\newpage

\appendix

\section{Proofs of Theorems in the Main Text}\label{sec:proofs-of-theorems-in-the-main-text}

\subsection{Proof of Theorem \ref{Surrogate Regret Bound for Linear Loss}}

\begin{proof}
\\
	From $P$ define $P_X$ to be the marginal distribution over instances and $\eta(x) = P(Y=1|X = x)$. Then,
	\begin{align*}
	\risklin(P,f) &= \EE_{(x,y) \dist P} 1 - y f(x) \\
	&= \EE_{x \dist P_X} 1 + (1-2\eta(x)) f(x).
	\end{align*}
	Minimizing over $f \in \BoundedF{X}{1}$ gives $f_{\mathrm{linear}, P}(x) = - 1$ if $1 - 2 \eta(x) \geq 0$ i.e. when $\eta(x) < \frac{1}{2}$ and $f_{\mathrm{linear}, P}(x) = 1$ otherwise. We have, 
	$$
	\risklin(P,f_{\mathrm{linear}, P}) = \EE_{x \dist P_X} 1 - \left|(1-2\eta(x))\right|.
	$$
	Therefore,
	\begin{align*}
	\risklin(P,f) - \risklin(P,f_{\mathrm{linear}, P}) &= \EE_{x \dist P_X} (1- 2 \eta(x)) f(x) + |(1-2\eta(x))| \\ 
	&= \EE_{x \dist P_X} \left|(1-2\eta(x))\right| - \sign(2 \eta(x) - 1) \left|(1-2\eta(x))\right| f(x) \\
	&= \EE_{x \dist P_X} \left|(1-2\eta(x))\right|(1 - \sign(2 \eta(x) - 1)f(x) ).
	\end{align*}
	It is well known that, 
	$$
	\riskmis(P,f) - \riskmis(P,f_{\mathrm{01}, P}) = \EE_{x \dist P_X} \left|(1-2\eta(x))\right| [\![ \sign(2\eta(x) - 1)f(x) \leq 0 ]\!].
	$$
	We complete the proof by noting $[\![ v \leq 0 ]\!] \leq 1 - v$ for $v \in [-1,1]$.
	
\end{proof}

\subsection{Proof of Theorem \ref{Means and Rademacher}}

Before the proof, we state a general form of McDiarmid's inequality, a well-known concentration of measure result.

\begin{theorem}[McDiarmid's inequality]
Let $Z_i$, $i\in [i; n]$, be a collection of $n$ independent random quantities each taking a value in some set $\Omega_i$, with $Z = \left(Z_1, Z_2, \dots, Z_n\right)$. Furthermore let $f : \times_{i=1}^n \Omega_i \rightarrow \RR$ with,
$$
c_i = \sup_{z, z' : z_j = z'_j \forall j \neq i} \left|f(z) - f(z') \right|.
$$
Then with probability at least $1 - \delta$, 
$$
f(z) \leq \EE f(Z) + \sqrt{\frac{\logdelta \sum_{i=1}^{n} c_i^2}{2}}.
$$

\end{theorem}
Intuitively, if the function $f$ is insensitive to perturbations in a single argument, and the arguments of $f$ can't ``conspire", then $f$ is concentrated around its expectation. We now prove theorem \ref{Means and Rademacher}.

\begin{proof}
Let $Z = \left((Y_1, X_1), ... , (Y_n, X_n))\right)$ and,
$$
f(z) = \norm{\Phi(P) - \frac{1}{n} \sum_{i=1}^n y_i \phi(x_i)} = \norm{\Phi(P) - \Phi(S)}.
$$
It is easily verified that $c_i = \frac{2}{n}$ for all $i \in [1,n]$. An application of McDiarmid's inequality yields,
$$
f(z) \leq \EE f(Z) + \sqrt{\frac{2 \logdelta}{n}}.
$$
with probability at least $1 - \delta$. All that remains is to bound $\EE f(z)$. We have,
\begin{align*}
	\EE f(Z) &= \EE \norm{\Phi(P) - \frac{1}{n} \sum_{i=1}^{n}Y_i \phi(X_i)} \\
	&\leq \sqrt{\EE \norm{\Phi(P) - \frac{1}{n} \sum_{i=1}^{n}Y_i \phi(X_i)}^2} \\
	&= \sqrt{\frac{1}{n^2} \sum_{i=1}^{n} \sum_{j=1}^{n}\EE \ip{\Phi(P) - Y_i \phi(X_i)}{\Phi(P) - Y_j \phi(X_j)}}\\
	&= \sqrt{\frac{1}{n^2} \sum_{i=1}^{n} \EE \norm{\Phi(P) - Y_i \phi(X_i)}^2} \\
	&\leq \frac{2}{\sqrt{n}},
\end{align*}
Where we have used the concavity of $\sqrt{\ }$, independence of the $\left(x_i, y_i\right)$ pairs and finally the boundedness of the feature map.
	
\end{proof}

\subsection{Proof of Theorem \ref{Margins and Approximation}}

Before the proof we prove the following simple lemma.

\begin{lemma}
	Let $v, \tilde{v} \in \RR$ with $|v - \tilde{v}| \leq \epsilon$. Then $\tilde{v} < 0$ implies $v < \epsilon$.
	
\end{lemma}

\begin{proof}	
	We have $v - \epsilon \leq \tilde{v} \leq v + \epsilon $. If $\tilde{v} < 0$, then $v - \epsilon < 0$.
	
\end{proof}
We now prove the theorem.

\begin{proof}
	By the conditions of the theorem, $| f(x) - \tilde{f}(x) | \leq \epsilon$ for all $x \in X$, meaning $| y f(x) - y \tilde{f}(x) | \leq \epsilon$ for all pairs $(x,y)$. By the previous lemma, $y \tilde{f}(x) < 0$ implies $y f(x) < \epsilon$. This means, 
	$$
	\pred{y \tilde{f}(x) < 0} \leq \pred{y f(x) < \epsilon}.
	$$
	Averaging over $P$ yields the desired result.
\end{proof}

\subsection{Proof of Proposition \ref{universal consistency}}

\begin{proof}
	Let $P(Y=1|X=x)$ be the conditional probability of observing the positive label. It is well known that the Bayes optimal classifier for misclassification loss is given by, $f_{01,P}(x) = 1$ if $P(Y=1|X=x) > \frac{1}{2}$ and $0$ otherwise.
	\\
	\\
	Let $P(\tilde{Y} = 1 | X = x)$ be the conditional probability of observing a positive label drawn from $P_\sigma$. By a simply calculation,
	\begin{align*}
		P(\tilde{Y} = 1 | X = x) &= (1 - \sigma)P(Y = 1 | X = x) + \sigma P(Y = -1 | X = x) \\
		&= (1 - 2 \sigma)P(Y = 1 | X = x) + \sigma,
	\end{align*}
	if $P(Y=1 | X = x) > \frac{1}{2}$ then,
	$$
	P(\tilde{Y} = 1 | X = x) > (1 - 2 \sigma)\frac{1}{2} + \sigma = \frac{1}{2}.
	$$
	Secondly, if $P(\tilde{Y} = 1 | X = x) > \frac{1}{2}$ then,
	$$
	(1 - 2 \sigma)P(Y = 1 | X = x) + \sigma > \frac{1}{2},
	$$
	which implies $P(Y = 1 | X = x) > \frac{1}{2}$. Therefore, $f_P^{01} = f_{P_\sigma}^{01}$.	
	
\end{proof}

\subsection{Proof of Theorem \ref{Strong Robustness to Label Noise Lemma}}

The proof requires the following result, which characterizes when two losses are order equivalent. 

\begin{proposition}[Theorem 2, section 7.9 \citep{DeGroot1962}]\label{De Groot}
	Let $\ell_1$ and $\ell_2$ be loss functions. $\ell_1$ and $\ell_2$ are equivalent in order if and only if there exist constants $\alpha > 0$ and $\beta$ such that, 
	$$
	\ell_2(y,v) = \alpha \ell_1(y,v) + \beta.
	$$
\end{proposition}
We now prove the theorem.

\begin{proof}	
	We begin with the reverse implication. Since, 
	$$
	\EE_{(v,y) \dist Q} \ell(y,v) = \EE_{(v,y) \dist Q_\sigma} \ell_{\sigma}(y,v),\ \forall Q, Q',
	$$
	we have $Q \leq_{\ell} Q' \Leftrightarrow Q_{\sigma} \leq_{\ell_\sigma} Q'_{\sigma}$. As we assume, $\ell$ and $\ell_\sigma$ are order equivalent, $Q_{\sigma} \leq_{\ell_\sigma} Q'_{\sigma} \Leftrightarrow Q_{\sigma} \leq_{\ell} Q'_{\sigma}$. Therefore,
	$$
	Q \leq_{\ell} Q' \Leftrightarrow Q_{\sigma} \leq_{\ell} Q'_{\sigma}.
	$$
	For the forward implication, define the loss $\ell'$ with,
	$$
	\TwoVector{\ell'(-1,v)}{\ell'(1,v)} = \Tsymmetrc \TwoVector{\ell(-1,v)}{\ell(1,v)},\ \forall v \in \RR.
	$$
	It is easily verified that $\ell'_\sigma = \ell$. This means,
	$$
	\EE_{(v,y) \dist Q} \ell'(y,v) = \EE_{(v,y) \dist Q_\sigma} \ell(y,v),\ \forall Q, Q',
	$$
	but as $Q \leq_{\ell} Q' \Leftrightarrow Q_{\sigma} \leq_{\ell} Q'_{\sigma}$, we have, 
	$$
	Q \leq_{\ell} Q' \Leftrightarrow Q \leq_{\ell'} Q'.
	$$
	Therefore $\ell$ and $\ell'$ are order equivalent. Invoking lemma \ref{De Groot} and the definition of $\ell'$ yields,
	$$
	\Tsymmetrc \TwoVector{\ell(-1,v)}{\ell(1,v)} =\alpha \TwoVector{\ell(-1,v)}{\ell(1,v)} + \beta \TwoVector{1}{1} ,\ \forall v \in \RR,
	$$
	for $\alpha > 0$. This yields,
	$$
	\TwoVector{\ell(-1,v)}{\ell(1,v)} =  \alpha \underbrace{\left( \Msymmetric \TwoVector{ \ell(-1,v)}{\ell(1,v)} \right)}_{\ell_\sigma} + \beta \TwoVector{1}{1} ,\ \forall v \in \RR.
	$$
	Therefore $\ell$ is order equivalent to $\ell_\sigma$.
	
\end{proof}

\subsection{Proof of Theorem \ref{Strong robustness to label noise characterization}}

\begin{proof}	
	As $\ell$ and $\ell_\sigma$ are equivalent in order, by the lemma \ref{De Groot}, $\ell_{\sigma}(y,v) = \alpha \ell(y,v) + \beta$. Combined with the definition of $\ell_\sigma$ yields,
	$$
	\frac{(1-\sigma) \ell(y,v) - \sigma \ell(-y,v)}{1- 2 \sigma} = \alpha \ell(y,v) + \beta.
	$$
	Setting $y = \pm 1$ yields the following two equations,
	\begin{align}
	(1-\sigma) \ell(1,v) - \sigma \ell(-1,v) &= (1- 2\sigma)(\alpha \ell(1,v) + \beta) \\
	(1-\sigma) \ell(-1,v) - \sigma \ell(1,v) &= (1- 2\sigma)(\alpha \ell(-1,v) + \beta).
	\end{align}
	Adding these two equations together and dividing through by $1-2 \sigma$ yields,
	\begin{equation}
	\ell(1,v) + \ell(-1,v) = \alpha (\ell(1,v) + \ell(-1,v)) + 2 \beta.
	\end{equation}
	If $\alpha \neq 1$, $\ell(1,v) + \ell(-1,v) = \frac{2 \beta}{1 - \alpha} = C$ and the proof is complete. If $\alpha = 1$, $\beta = 0$ by $(3)$. Inserting these values into $(2)$ yields,
	$$
	(1-\sigma) \ell(1,v) - \sigma \ell(-1,v) = (1- 2\sigma)\ell(1,v).
	$$
	Thus $\ell(1,v) = \ell(-1,v)$, an excluded pathological case. For the converse, if $\ell(y,v) + \ell(-y,v) = C$ then $\ell(-y,v) = C - \ell(y,v)$. This means,
	\begin{align*}
	\ell_{\sigma}(y,v) &= \frac{(1-\sigma) \ell(y,v) - \sigma \ell(-y,v)}{1- 2 \sigma} \\
	&= \frac{(1-\sigma) \ell(y,v) - \sigma (C - \ell(y,v))}{1- 2 \sigma} \\
	&= \frac{1}{1 - 2 \sigma} \ell(y,v) - \frac{\sigma C}{1 - 2 \sigma},
	\end{align*}
	and thus by the above lemma, $\ell$ and $\ell_\sigma$ are equivalent in order.
	
\end{proof}

\subsection{Proof of Theorem \ref{Uniqueness of Linear Loss}}

\begin{proof}	
	We begin with the forward implication. We have $\ell(y,v)$ is convex in $v$, furthermore $\ell(y,v) + \ell(-y,v) = C$. This means $\ell(y,v) = C - \ell(-y,v)$, hence $-\ell(-y,v)$ is convex. Thus as $\ell(y,v)$ and $-\ell(y,v)$ are convex, $\ell(y,v) = \alpha_y v + g(y)$. But, 
	\begin{align*}
	\ell(y,v) + \ell(-y,v) &= \alpha_y v + g(y) + \alpha_{-y} v + g(-y) \\
	&= (\alpha_y + \alpha_{-y}) v + g(y) + g(-y) \\
	&= C.
	\end{align*}
	Therefore $\alpha_{-y} = - \alpha_{y} = \lambda$ and $\ell(y,v) = \lambda y v + g(y)$. For the converse, if $\ell(y,v) = \lambda y v + g(y)$, then, 
	$$
	\ell(y,v) + \ell(-y,v) = g(y) + g(-y) = C.
	$$
	Therefore any loss that is convex in its second argument and robust to label noise is order equivalent to,
	$$
	\ell(y,v) = \lambda y v.
	$$
	By the characterization of classification calibration \citep{Bartlett2006}, we must have $\lambda < 0$ for $\ell$ to be classification calibrated.

\end{proof}

\subsection{Proof of Theorem \ref{No loss is immune to class conditional label noise}}

\begin{proof}	
	If $\sigma_{1} \ell(-1,v) + \sigma_{-1} \ell(1,v) = C $, this means $\sigma_{-y} \ell(y,v) + \sigma_{y} \ell(-y,v) = C $ for all $y$. This yields,
	\begin{align*}
	\ell_{\sigma_{-1}, \sigma_1}(y,v) &= \frac{(1-\sigma_{-y}) \ell(y,v) - \sigma_y \ell(-y,v)}{1- \sigma_{-1} - \sigma_1} \\
	&= \frac{(1-\sigma_{-y}) \ell(y,v) - (C - \sigma_{-y} \ell(y,v))}{1- \sigma_{-1} - \sigma_1} \\
	&= \frac{1}{1- \sigma_{-1} - \sigma_1} \ell(y,v) - \frac{C}{1- \sigma_{-1} - \sigma_1}, 
	\end{align*}
	where the first line is the definition of $\ell_{\sigma_{-1}, \sigma_1}(y,v)$ and the second is by assumption. By lemma \ref{De Groot}, $\ell_{\sigma_{-1}, \sigma_1}$ and $\ell$ are order equivalent.
	
\end{proof}

\subsection{Proof of Theorem \ref{Balanced Error Noise Immunity}}

\begin{proof}
	Recall the balanced error,
	$$
	\mathrm{BER}_{\ell}(P_+,P_-,f) = \frac{1}{2} \EE_{x\dist P_+} \ell(1,f(x)) + \frac{1}{2} \EE_{x \dist P_-} \ell(-1,f(x)).
	$$
	Remember that,
	\begin{align*}
	\tilde{P}_+ = (1 - \alpha) P_+ + \alpha P_-  \ \text{and}\ & \tilde{P}_- = \beta P_+ + (1 - \beta) P_-.
	\end{align*}
	This means for all classifiers $f$,
	\begin{align*}
	\EE_{x \dist \tilde{P}_+} \ell(1,f(x)) &= (1 - \alpha) \EE_{x \dist P_+} \ell(1,f(x)) + \alpha \EE_{x \dist P_-} \ell(1,f(x)) \\
	&= (1 - \alpha) \EE_{x \dist P_+} \ell(1,f(x)) - \alpha \EE_{x \dist P_-} \ell(-1,f(x)) + C \alpha,
	\end{align*}
	where in the second line we have used the fact that $\ell(1,v) = C - \ell(-1,v)$. Similarly,
	$$
	\EE_{x \dist \tilde{P}_-} \ell(-1,f(x)) = -\beta \EE_{x \dist P_+} \ell(1,f(x)) + (1-\beta) \EE_{x \dist P_-} \ell(-1,f(x)) + C \beta.
	$$
	Taking the average of these two equations yields,
	$$
	\mathrm{BER}_{\ell}(\tilde{P}_+,\tilde{P}_-,f) = (1 - \alpha - \beta) \mathrm{BER}_{\ell}(P_+,P_-,f) + \frac{\left(\alpha + \beta \right)}{2} C.
	$$
	
\end{proof}

\subsection{Proof of Theorem \ref{robustness under all noise: theorem}}

\begin{proof}	
	Firstly, for all classifiers $f$,
	\begin{align*}
	\riskL{\ell}(P_\sigma,f) &= \EE_{(x,y)\dist P} (1-\sigma(x,y)) \ell(y,f(x)) + \sigma(x,y) \ell(-y,f(x)) \\
	&= \EE_{(x,y)\dist P} (1-\sigma(x,y)) \ell(y,f(x)) + \sigma(x,y) (C - \ell(y,f(x))) \\
	&= \EE_{(x,y)\dist P} (1-2 \sigma(x,y)) \ell(y,f(x)) + C \EE_{(x,y)\dist P} \sigma(x,y), 
	\end{align*}
	where in the second line we have used the fact that $\ell(1,v) + \ell(-1,v) = C$. Now let,
	\begin{align*} 
	f^*_\sigma = \argmin_{f \in \Fclass} \ell (P_\sigma, f) \ \text{and}\ & f^* = \argmin_{f \in \Fclass} \ell (P, f), 
	\end{align*}
	respectively. By definition, $\ell(P_\sigma, f^*_\sigma) \leq \ell(P_\sigma, f^*)$. Combined with the above this yields,
	$$
	\EE_{(x,y)\dist P} (1-2 \sigma(x,y)) \ell(y,f^*_\sigma(x)) \leq \EE_{(x,y)\dist P} (1-2 \sigma(x,y)) \ell(y,f^*(x)).
	$$
	From the assumption that $\sigma(x,y) < \frac{1}{2}$ for all $(x,y) \in X \times Y$,
	$$
	\min_{(x,y)}1-2 \sigma(x,y) \leq  1-2 \sigma(x,y) \leq 1 ,\ \forall (x,y) \in X \times Y.
	$$
	This yields,
	$$
	\left(\min_{(x,y)}1-2 \sigma(x,y)\right) \EE_{(x,y)\dist P} \ell(y,f^*_\sigma(x)) \leq \EE_{(x,y)\dist P}  \ell(y,f^*(x)),
	$$
	and the proof is complete.
	
\end{proof}

\subsection{Proof of Theorem \ref{Sparse Approximation Lower Bound}}

\begin{proof}
	Let $\left\{e_i\right\}_{i=1}^\infty$ be an orthonormal basis for $\Hilbert$, $\ip{e_i}{e_j} = 1$ if $i=j$ and $0$ otherwise. Fix $n > 0$ and let $S = \left\{e_i\right\}_{i=1}^n$ with $\omega^* = \frac{1}{n} \sum_{i=1}^{n}e_i$. It is easily verified that,
	$$
	\omega^* = \argmin_{\omega \in \chull(S)} \norm{\omega}^2,
	$$
	furthermore $\norm{\omega^*}^2 = \frac{1}{n}$. Lemma 3 of \citep{Jaggi2013} states for all $k$-sparse $\tilde{\omega}\in \chull(S)$, $\norm{\tilde{\omega}}^2 \geq \frac{1}{k}$. Therefore,
	$$
	\norm{\tilde{\omega}}^2 - \norm{\omega^*}^2 \geq \frac{1}{k} - \frac{1}{n},
	$$ 
	for all $k$-sparse $\tilde{\omega}$. Note that $\omega^*$ is the orthogonal projection of $0$ onto $\chull(S)$. Therefore by the Pythagorean theorem, $\norm{\tilde{\omega}}^2 - \norm{\omega^*}^2 = \norm{\omega^* - \tilde{\omega}}^2$, yielding,
	$$
	\norm{\omega^* - \tilde{\omega}} \geq \sqrt{\frac{1}{k} - \frac{1}{n}},
	$$
	and the claim is proved.
	
\end{proof}

\subsection{Proof of Theorem \ref{Binning and sub-sampling for fun and profit}}

\begin{proof}
	For the first claim, denote by $l_i$ the sparsity of $\omega_i$ and by $l$ the sparsity of $\omega$. We have,
	$$
	l = \sum_{i = 1}^{m} l_i \leq \sum_{i = 1}^{m} \lceil \alpha(S_i) m \rceil \leq \sum_{i = 1}^{m} \alpha(S_i) m + 1 = 2 m.
	$$
	where the first inequality holds as their may be repeated elements in the sub-sample, and the second follows from the definition of ceiling. For the second claim, considering the collection of independent random quantities $Z_{ij} \dist P_i$, $P_i$ is the distribution with support $S_i$ and $\omega \in S_i$ is chosen with probability $\frac{\alpha(\omega)}{\alpha(S_i)}$. Define,
	\begin{align*}
	Z_i  &= \frac{1}{\lceil \alpha(S_i) m \rceil} \sum_{j=1}^{\lceil \alpha(S_i) m \rceil} Z_{ij}\\
	Z &= \sum_{i=1}^{m} \alpha(S_i) Z_i.
	\end{align*}
	It is easily verified that,
	\begin{align*}
	\EE Z_i  &= \EE Z_{ij} = \sum_{\omega \in S_i} \frac{\alpha(\omega)}{\alpha(S_i)} \omega\\
	\EE Z &= \sum_{i=1}^{m} \alpha(S_i)\EE Z_i = \sum_{\omega \in S} \alpha(\omega) \omega.
	\end{align*}
	Here we use McDiarmid's Inequality to control variations of $\norm{\EE Z - Z}$. Firstly, by construction of the partition,
	$$
	c_{ij} \leq \frac{D}{m}.
	$$
	An application of McDiarmid's inequality yields
	\begin{align*}
	\norm{\EE Z - Z} &\leq \EE \norm{\EE Z - Z} + \sqrt{\frac{\logdelta \sum_{i=1}^{m} \sum_{j=1}^{\lceil \alpha(S_i) m\rceil } c_{ij}^2}{2}} \\
	& \leq \EE \norm{\EE Z - Z} + \sqrt{\frac{\logdelta \sum_{i=1}^{m} \sum_{j=1}^{\lceil \alpha(S_i) \rceil m} \frac{D^2}{m^2}}{2}} \\
	&\leq \EE \norm{\EE Z - Z} + \sqrt{\frac{\logdelta D^2}{m}},
	\end{align*}
	where the second line follows from the bound on $c_{ij}$ and the third follows as there are at most $2m$ terms in the summation. All that remains is to bound $\EE \norm{\EE Z - Z}$. As in the proof of theorem \ref{Means and Rademacher},
	\begin{align*}
	\EE \norm{\EE Z - Z} & \leq \sqrt{\EE \norm{\EE Z - Z}^2} \\
	&= \sqrt{\sum_{i = 1}^{m} \sum_{i' = 1}^{m} \EE \ip{\alpha(S_i)\left(\EE Z_i - Z_i\right)}{\alpha(S_{i'})\left(\EE Z_{i'} - Z_{i'}\right)}} \\
	&= \sqrt{\sum_{i = 1}^{m} \alpha(S_i)^2 \EE \norm{\EE Z_i - Z_i}^2} \\
	&= \sqrt{\sum_{i = 1}^{m} \frac{\alpha(S_i)^2 }{\lceil m \alpha(S_i) \rceil}\EE \norm{\EE Z_{ij}  - Z_{ij}}^2},\ \forall j \\
	&\leq \sqrt{\sum_{i = 1}^{m} \frac{\alpha(S_i) }{m} D^2} \\
	&\leq \frac{D}{\sqrt{m}},
	\end{align*}
	Where we have used the concavity of $\sqrt{\ }$, the independence of the $Z_i$, that fact $Z_i$ is the sum of $\lceil m \alpha(S_i) \rceil$ iid random quantities and then finally a bound on the variance of $Z_{ij}$ in terms of the diameter of the partition coupled with the fact $ m \alpha(S_i) \leq \lceil m \alpha(S_i) \rceil$.
	
\end{proof}

\bibliographystyle{plainnat}
\bibliography{./References}

\end{document}